\documentclass[mnsc,nonblindrev]{informs3}
\usepackage{booktabs} % For formal tables

\usepackage[ruled]{algorithm2e} % For algorithms

\SetAlFnt{\small}
\SetAlCapFnt{\small}
\SetAlCapNameFnt{\small}
\SetAlCapHSkip{0pt}
\IncMargin{-\parindent}
\OneAndAHalfSpacedXI

\usepackage{tcolorbox}
\usepackage{bbm}
\usepackage{ulem}

\newcommand{\CFE}{{\sc CaFE}}
\newcommand{\Reg}{\textup{Regret}}
\newcommand{\KL}{\textup{KL}}

\newcommand{\cC}{\mathcal{C}}

\newcommand{\cF}{\mathcal{F}}

\newcommand{\cD}{\mathcal{D}}
\newcommand{\cU}{\mathcal{U}}
\newcommand{\cX}{\mathcal{X}}

\newcommand{\cW}{\mathcal{W}}

\usepackage{natbib}
 \bibpunct[, ]{(}{)}{,}{a}{}{,}%
 \def\bibfont{\small}%
\TheoremsNumberedThrough     % Preferred (Theorem 1, Lemma 1, Theorem 2)
\ECRepeatTheorems

\EquationsNumberedThrough    % Default: (1), (2), ...
\MANUSCRIPTNO{}

\begin{document}
\RUNAUTHOR{Gupta and Kamble}
\RUNTITLE{Individual Fairness in Hindsight}  
\TITLE{Individual Fairness in Hindsight}  
\ARTICLEAUTHORS{%
\AUTHOR{Swati Gupta}
\AFF{Georgia Institute of Technology, Atlanta, GA\\ \EMAIL{swatig@gatech.edu}}
\AUTHOR{Vijay Kamble}
\AFF{University of Illinois at Chicago, Chicago, IL\\ \EMAIL{kamble@uic.edu}} %, \URL{}}
% Enter all authors
} % e
\ABSTRACT{
Since many critical decisions impacting human lives are increasingly being made by algorithms, it is important to ensure that the treatment of individuals under such algorithms is demonstrably fair under reasonable notions of fairness. One compelling notion proposed in the literature is that of individual fairness (IF), which advocates that similar individuals should be treated similarly \citep{Dwork2012}. Originally proposed for offline decisions, this notion does not, however, account for temporal considerations relevant for online decision-making. In this paper, we extend the notion of IF to account for the {\it time} at which a decision is made, in settings where there exists a notion of conduciveness of decisions as perceived by the affected individuals. We introduce two definitions: (i) fairness-across-time (FT) and (ii) fairness-in-hindsight (FH). FT is the simplest temporal extension of IF where treatment of individuals is required to be individually fair relative to the past as well as future, while in FH, we require a one-sided notion of individual fairness that is defined relative to {\it only} the past decisions. We show that these two definitions can have drastically different implications in the setting where the principal needs to learn the utility model. Linear regret relative to optimal individually fair decisions is inevitable under FT for non-trivial examples. On the other hand, we design a new algorithm: {\it Cautious Fair Exploration} (\CFE), which satisfies FH and achieves sub-linear regret guarantees for a broad range of settings. We characterize lower bounds showing that these guarantees are order-optimal in the worst case. FH can thus be embedded as a primary safeguard against unfair discrimination in algorithmic deployments, without hindering the ability to take good decisions in the long-run.
}
\maketitle

\section{Introduction}\label{sec:intro}
 Algorithms facilitate decisions in increasingly critical aspects of modern life -- ranging from search, social media, news, e-commerce, finance, to determining credit-worthiness of consumers, estimating a felon's risk of reoffending, determining candidacy for clinical trials, etc. 
Their pervasive prevalence has motivated a large body of scientific literature in the recent years that examines the effect of automated decisions on human well-being, and in particular, seeks to understand whether these effects are {\it fair} under various notions of fairness \citep{Dwork2012, Sweeney2013, kleinberg2016inherent, Angwin2016, Hardt2016, chouldechova2017fair, chouldechova2017fairer, Corbett2018} . 

In this context of automated decisions, fairness is often considered in a relative sense rather than an absolute sense.  In his 1979 Tanner lectures, economist Amartya Sen eloquently argued that the heart of the issue rests on clarifying the ``equality of what?" problem \citep{sen20136}.
%In his 1979 Tanner Lectures, Amartya Sen noted that since nearly all theories of fairness are founded on an equality of some sort, the heart of the issue rests on clarifying the ``equality of what?" problem (\cite{Hu2018}, and references therein). 
Equality can be desired with respect to opportunity \citep{Hardt2016}, outcomes \citep{phillips2004defending}, treatment \citep{Dwork2012}, or even mistreatment \citep{zafar2017fairness}. In this paper, we consider the notion of {\it individual fairness} \citep{Dwork2012}, that relies on the premise of equality of treatment, requiring that ``similar'' individuals must be treated ``similarly''. This intuitively compelling  notion of fairness was proposed in the influential work of \cite{Dwork2012} in the context of classification in supervised learning, and has since been studied under several settings \citep[see][]{Yona2018,Dwork2018individual, heidari-2018}. The key idea is to introduce a Lipschitz condition on the decisions of a classifier, such that for any two individuals $x$, $y$ that are at distance $d(x, y)$, the corresponding distributions over decisions $M(x)$ and $M(y)$ are also statistically close within a distance of some multiple of $d(x, y)$. 
%\textcolor{red}{Individual similarity is a simple yet powerful notion that mathematically models equality in treatment by processes, as mandated by the legal framework \textcolor{red}{[ref]}.}[VJ: I don't understand this line...]

%The notion of individual fairness, however, is too restrictive in the context of online decision-making where an algorithm operates under partial information and the true state of the system needs to be learned over time. Learning inevitably requires some form of exploration, i.e., observing the consequence of different decisions made for possibly similar individuals. Individual fairness inherently prohibits this. To resolve this dichotomy between restrictions of individual fairness and operability of real-time online learning methods, 
%The notion of individual fairness, however, is a static notion o
%
Individual fairness, as originally defined, is a static notion that pertains to offline or batch decisions. Considering the fact that many algorithms for automated decision-making are sequential in nature, in this paper, we propose an extension of individual fairness that explicitly accounts for the time at which decisions are made. We specifically focus on settings where there exists a common notion of conduciveness of decisions from the perspective of the individuals affected by these decisions; e.g., approval of a higher loan amount is more conducive to a loan applicant than the approval of a smaller amount, a shorter jail term is more conducive from the perspective of a convict, lower prices are more conducive for shoppers, etc. 
%In doing so, we draw from a rich set of related considerations in existing societal systems (refer to Section \ref{sec:motivation}). 

As a motivating example, suppose two similar\footnote{We discuss issues with defining such similarity in Section~\ref{sec:motivation}, but for now assume that this means two persons with exactly same observable attributes.}  persons A and B apply for a loan at a bank and the bank approves a substantially higher loan amount to B than to A. This would be perceived as unfair (in the vague, colloquial sense) by A, but maybe not by B. Under the classical definition of individual fairness, this distinction is irrelevant -- implicitly, it is sufficient that either of the two similar individuals find a drastically different treatment problematic, and hence the loan amounts approved for A and B must be similar.  

The introduction of time allows for a richer treatment of the case above. For instance, if A applied for the loan earlier than B, then the treatment of A can still be defined to be fair as long as A got approved of a loan that is (approximately) at least as much as that approved for similar people {\it who applied before her}. In other words, A's treatment by the bank can be deemed to be fair solely based on the history of decisions at the time when the loan was approved, despite the fact that this treatment turns out to be ``individually unfair'' in retrospect when B later gets approved for a substantially higher amount. 

Armed with this basic intuition, we define {\it fairness-in-hindsight}: decisions are said to be fair in hindsight if the decisions for incoming individuals are individually fair relative to the past decisions for similar individuals, in the sense that they become more conducive over time by respecting a certain lower bound for rewards or an upper bound for penalties. To contrast this notion and to serve as a baseline, we also consider a more straightforward temporal extension of individual fairness, which we call fairness-across-time, where the treatments of individuals are required to be individually-fair relative to the past as well as the future. This means that irrespective of when they arrive, similar individuals must always receive similar treatment: neither more conducive nor less conducive than what is justified by their degree of dissimilarity.

Our main technical contribution of this paper is to study the implications of these fairness constraints in situations where an algorithm operates under partial information and the true utility model needs to be learned over time. A standard performance metric in such sequential decision-making settings featuring learning is the {\it regret} incurred by an {\it oblivious} algorithm. This is defined as the difference between the optimal utility if the underlying utility model was known, and the utility under the algorithm. It is desirable that an algorithm ensures that the average regret in the long-run converges to 0, i.e., informally, it eventually learns and settles on good decisions. 

The notion of treating similar individuals similarly, {\it always}, seems restrictive in such uncertain settings where it is a priori unclear which decisions are good. The large body of literature on sequential decision-making under uncertainty has revealed that a certain degree of experimentation, at least in the early stages, may be fundamentally necessary to learn good decisions in the long run. But under the fairness-across-time constraint, bad decisions made in the early stages of experimentation may have to be repeated forever. We show that indeed, this feature typically has dire consequences on regret: under fairness-across-time, except for trivial settings, the expected regret against the benchmark of an optimal individually fair decision-rule grows linearly with the decision horizon. This is illustrated in the example below.

%The notion of fairness-across-time is too restrictive in the context of online decision-making where an algorithm operates under partial information and the true state of the system needs to be learned over time. 
%Learning inevitably requires some form of exploration, i.e., observing the consequence of different decisions made for possibly similar individuals. Individual fairness inherently prohibits this. 

%[The notion of treating similar individuals similarly, {\it always}, seems restrictive in uncertain settings where the functions that map decisions to the utility of the decision-maker for different individuals are unknown, and hence it is a-priori unclear which decisions are good. The large body of literature on decision-making under uncertainty has revealed that a certain degree of experimentation, at least in the early stages, may be fundamentally necessary to learn good decisions in the long run.] 

\begin{example}\label{ex1}
{\it Consider a bank making loan approval decisions in a new market that it has just entered. The decision space is $\cX = [0,1]$ representing the amount of loan sanctioned (normalized to 1). 
Consider a hypothetical setting where the only feature that bank observes is the age of the applicant, and it does not a-priori know whether age is positively or negatively correlated with default probability. If the first applicant A is young, say 18 years old, and is given a loan of amount $\mathdollar$M, then any future applicant B aged 18 must be given $\mathdollar$M to satisfy fairness-across-time. But this decision of $\mathdollar$M loan is bound to be suboptimal and leads to a linear regret when M is small but age is positively correlated with default probability or when M is large and age is negatively correlated with default probability. }
\end{example}

In contrast, we demonstrate that the situation is not as pessimistic under fairness-in-hindsight: the possibility that decisions can become more conducive over time gives a powerful leeway that allows algorithms to learn {\it and} settle on good decisions over time. Formally, we design an algorithm that we call {\it Cautious Fair Exploration} ({\sc CaFE}), which is individually fair in hindsight, and attains sub-linear regret guarantees as compared to the optimal individually-fair benchmark in a wide range of settings. 

{\sc CaFE} operates in two phases, exploration and exploitation. In the exploration phase, the decisions are conservative and the goal is to learn the utility model. Once the utility model is learned with the appropriate confidence, the algorithm enters the exploitation phase where the decisions are then allowed to become more conducive for the appropriate individuals while ensuring individual fairness. The following example illustrates this point.

\begin{example}
{\it Consider Example~\ref{ex1}. In this case, the bank can approve a small amount of loan, say $\epsilon$, to each applicant in an initial exploration phase. Once the bank learns the correlation structure with appropriate confidence, it can start approving loans for larger amounts as appropriate in an individually-fair manner, while guaranteeing a loan of $\epsilon$ to everyone. This ensures fairness-in-hindsight.}
\end{example}

%The structure of {\sc CaFE} is intuitive and bears resemblance to some of the features observed in existing societal systems. For instance, it is typical for conservative legal stances on new issues to become more liberal over time, e.g., decriminalization laws remove penalties for actions perceived as crimes in the past and such amendments are perceived as more conducive and fair, at least in hindsight. [Add more examples]

{\sc CaFE} critically relies on the ability to learn with conservative decisions. But in many situations, learning is slow in the conservative regime. For instance, in the examples above, when the loan amount $\epsilon$ is small, the default probability is expected to be small irrespective of the age of the individual. Hence, learning the correlation structure might take prohibitively long. On the other hand if $\epsilon$ is large, then after having learned the model, the bank is forced to approve an amount of at least $\epsilon$ to individuals with a high default probability, leading to high regret. Thus, there is a fundamental tradeoff between conservatism and the learning rate that is relevant to the overall regret incurred. Our technical results shed a light on how decisions should be chosen in the exploration phase so as to balance this tradeoff with the goal of minimizing regret. Our sublinear upper bounds on the regret of the resulting algorithm are accompanied by matching lower bounds %(in the exponent of the decision horizon) 
that justify our design.

%The structure of {\sc CaFE} is intuitive and bears resemblance to some of the features observed in existing societal systems. For instance, it is typical for conservative legal stances on new issues to become more liberal over time, e.g., decriminalization laws remove penalties for actions perceived as crimes in the past and such amendments are perceived as more conducive and fair, at least in hindsight. [Add more examples: fair pricing]

The paper is organized as follows. We build on the motivation for fairness-in-hindsight in the next section and survey relevant literature in Section \ref{sec:related}. We then present the model and the fairness definitions in Section~\ref{sec:model}. We next present the learning problem in Section~\ref{sec:dynamic}, where we present most of our main technical results. We conclude the paper in Section~\ref{sec:conclusion}.

\section{Motivation and Discussion}\label{sec:motivation} 
The notion of individual fairness over time is already ingrained in many of our societal systems, where it provides a dynamic, tangible frame of reference for our sense of social justice: 
%The notion of fairness-in-hindsight provides a dynamic, tangible frame of reference for our sense of social justice and it is ingrained in many of our societal systems. 
%These notions provide a tangible frame of reference for our sense of social justice. 
%Directionality of constraints on decisions in time has been acknowledged by the law wherein laws can be applied prospectively (affect decisions of future cases) or retrospectively (affect decisions of pending or past cases) \cite{friedland1974prospective}.  
\begin{enumerate}
\item {\bf Law: } In many legal systems, once a precedent is set by a ruling, decisions for similar cases observed in the future must follow these precedents. %\citep{fowler2008authority} \sw{ref}. 
This is often referred to as {\it stare decisis}, i.e., to stand by things decided.\footnote{See \url{https://www.law.cornell.edu/wex/stare_decisis}.} Further, new laws affect future decisions (prospective application of law), whereas they seldom change the decisions of past rulings (retrospective application of law) \citep{friedland1974prospective}. The effect of a new law is therefore not symmetric with respect to past and future decisions, thus bearing similarities to our unidirectional notion of fairness-in-hindsight.
%Further, rulings at any point in time affect future decisions by setting a precedent (prospective application of law), whereas they seldom change the decisions of past rulings (retrospective application of law)\citep{friedland1974prospective}. The effect of a ruling is therefore not symmetric with respect to past and future decisions, thus bearing similarities to our unidirectional notion of fairness-in-hindsight. %Fairness-in-hindsight, in some sense, models the same notion of comparative fairness with respect to decisions in the past. 
%Decriminalization laws remove penalties for actions perceived as crimes in the past and such amendments are perceived as only being fair and more conducive, in line with our notion of fairness-in-hindsight.\footnote{\url{https://www.nytimes.com/2018/09/06/world/asia/india-gay-sex-377.html}} 
\item{\bf Whataboutism: } There has been a growing culture of {\it whataboutism}, which is an argument device used to prove mistreatment by pointing to another similar context that received a different, more conducive, treatment in the past. For example, politicians might defend their questionable actions or advocate no-penalty since a prior political scandal of an opposition party did not go through due process or face significant consequences.\footnote{A common argument used by the defenders of the republican US President Richard Nixon's administration after the Watergate scandal in the United States was to point to an older unfortunate case of the Chappaquiddick accident that democrat Ted Kennedy was implicated in; see \url{https://bit.ly/2ORRj8L}. A similar instance is the ``What about Hillary?" rhetoric used by supporters of US President Donald Trump; see \url{https://bit.ly/2GHRCC8}.}%With fairness-in-hindsight, we build a strong precedence of prior contexts that restrict future decisions on similar contexts so that whataboutisms can be effectively be combated. 

\item {\bf Pricing: } 
%In retail or e-commerce settings, price experimentation to learn demand is often perceived as unfair as similar consumers can be charged different prices across time (see \cite{Bird2016} for a related discussion). 
\cite{bolton2003consumer} find that customers' impression of fairness of prices critically relies on past prices that act as reference points. In particular, they find that any price increases beyond that justified by increase in perceived costs, e.g., due to inflationary effects, are perceived as price gouging and unfair by the customers, and moreover to make matters worse for the firm, customers tend to underestimate these costs.

%The general data protection regulation, a regulation in the EU law that came into effect in May 2016, outlines the rights of any individual for whom a decision is made by an automated mechanism. These rights include the right to contest any automated decision-making that was made on a solely algorithmic basis (Article 12 (3)). Fairness-in-hindsight constrains automated decisions to the decisions made for past contexts, thereby empowering the decision-maker to explain their decisions for any individual using precedence. 
\end{enumerate}

Given the prevalence of this notion in practice, it is only natural to expect our algorithms to uphold these same principles. A natural question that arises then is:
what are the implications of these constraints on the long-run performance of algorithms that typically operate under system uncertainty? As we show through \CFE, unlike the more stringent notion of fairness-across-time under which high regret is typically inevitable, fairness-in-hindsight does allow us to learn and settle on good decisions over time. The conservative exploration then exploitation structure of {\sc CaFE} is intuitive and bears resemblance to certain features seen in existing societal systems. For instance, it is typical for legal stances on new issues to be more conservative initially and then potentially become more liberal over time as the impact and nuances of these issues become clear, e.g., decriminalization laws remove penalties for actions perceived as crimes in the past. Another example is that of markdowns in retail where the prices of new goods that did not see anticipated demand are decreased over time. 

\paragraph{The issue of the metric.}
%Two broad camps have emerged in the realm of fairness in machine learning. 
%One of the early goals of research in fairness in machine learning was to ensure that treatments of an algorithm should not be biased against a specific subgroup either directly or indirectly, despite the fact this requirement could potentially be at odds with the decision-maker's utilities (thus amounting to a form of {\it affirmative action}).
%, the implicit goal being to attempt to reverse the effects of inequities and biases that some of these subgroups have been exposed to historically or at least stop the perpetuation of these biases. 
%A natural notion that comes out of this requirement is {\it statistical parity} or {\it Group Fairness}: the treatment should not reveal any information about group membership \cite{klienberg, corbett-davies, Dwork} . This is equivalent to saying that the overall treatment distributions for any subgroup are the same as that of the entire population. 
%There are certain concerns with Group Fairness, mainly that there are easy yet obviously unfair ways of ensuring it. A commonly cited example is a company interviewing candidates from different subgroups in proportion to their representation in the population -- hence satisfying Group Fairness -- but selecting obviously unqualified candidates from all but one preferred subgroup. 
%This issue is mitigated with the additional requirement of
\cite{Dwork2012} admit that the existence and availability of a similarity metric between individuals 
%(with certain observable attribute-sets) 
for a particular decision-making problem is one of the most challenging aspects of the notion of individual fairness. 
%and they discuss various possibilities for addressing this concern. 
\begin{quote}
{\it Our approach is centered around the notion of a task-specific similarity metric describing the extent
to which pairs of individuals should be regarded as similar for the classification task at hand. The
similarity metric expresses ground truth. When ground truth is unavailable, the metric may reflect the
``best'' available approximation as agreed upon by society. Following established tradition \citep{rawls2001justice} the
metric is assumed to be public and open to discussion and continual refinement. Indeed, we envision
that, typically, the distance metric would be externally imposed, for example, by a regulatory body, or
externally proposed, by a civil rights organization. }\citep{Dwork2012} 
\end{quote}
In the context of automated decisions, individuals typically appear as a vector of attributes to an algorithm and hence in principle, {\it some} metric is readily available. But the issue is that it is a priori unclear which attributes of an individual should be considered to be relevant to a decision-making task from a fairness perspective, what is the relative importance of these attributes, which attributes must be ignored completely, etc. This choice is especially non-trivial because there could be seemingly non-controversial attributes, e.g., preferred genre of music, education, zipcode, etc., that are correlated with membership in protected population subgroups, which could be the basis for disparate treatment under an individually fair algorithm \citep{pedreshi2008discrimination}.
%Importantly, the problem of consensus on these concerns is compounded by the fact that the notion of Individual Fairness can be at odds with another natural notion of fairness called group fairness aka statistical parity, which says that the overall treatment distributions for any subgroup must be the same as that of the entire population \cite{Dwork2012}. This is because there could be attributes that are correlated with group membership, which could be the basis for discrimination under an individually fair algorithm.\footnote{Although, a compelling resolution of this concern was offered in \cite{zemel2013learning} who demonstrated that Individual Fairness and Group Fairness can be achieved together by transforming the attribute-set of an individual into a different representation that anonymizes group membership (hence ensuring Group Fairness), relative to which one can design an individually fair algorithm.}

We acknowledge that our proposal of fairness-in-hindsight inherits these concerns. However, one compelling resolution of these concerns was offered by  \cite{zemel2013learning}, who show how to transform the attribute-set of an individual into a different representation that anonymizes sub-group membership, relative to which one can design an individually fair algorithm. Additionally, we argue that satisfying fairness-in-hindsight relative to {\it some} reasonable metric, and which ideally is publicly announced, is an imperative in the current times. For example, the general data protection regulation (GDPR), a regulation in the EU law that came into effect in May 2016, outlines the rights of any individual for whom a decision is made by an automated mechanism. These rights include the right to contest any automated decision-making that was made on a solely algorithmic basis (Article 21). A body employing automated decision-making that does not satisfy fairness-in-hindsight relative to some reasonable metric can be highly susceptible to claims of demonstrable discrimination. Fairness-in-hindsight is a basic yet necessary safeguard against such claims.

\section{Related Literature}\label{sec:related}
%\subsection{Fairness in Online Learning}
Early research on fairness in machine learning focused on the offline setting of batch supervised learning from observational data \citep{pedreshi2008discrimination, kamiran2009classifying, calders2010three, kamishima2011fairness,  Dwork2012, Hardt2016, kleinberg2016inherent}. 
%where the primary goal is to mitigate the issue of \emph{disparate impact} in automated decisions: which refers to practices that collectively allocate a more favorable outcome to one population group compared to another.
Only recently has the literature started looking at the implications of fairness constraints in online learning settings. Our model and results complement a small but growing body of literature in this domain \citep{joseph2016fairness, liu2017calibrated, gillen2018online, heidari-2018, celis2018algorithmic, joseph2016fair, elzayn2018fair}. The two papers most related to our work are \cite{joseph2016fairness} and \cite{heidari-2018} that we discuss below. 

\cite{joseph2016fairness} were one of the earliest to study the impact of fairness constraints on learning in a contextual multi-armed bandit setting under a utility maximization objective. They proposed a {\it meritocratic} notion of fairness so that with high probability over the entire decision horizon, the probability of picking an arm (i.e., a subgroup) at any time is monotonic in its underlying mean reward. Their notion of fairness is however limited to individuals that appear {\it within} in each time period, whereas our notion of fairness is relative to all those who arrived in the past.%Further, meritocratic fairness may often not be well-aligned with individual fairness, i.e. the notion that similar individuals must be treated similarly, which is the premise of this paper.} 

In a recent work by \cite{heidari-2018}, the authors extend individual fairness to account for the notion of time and study its impact on learning. They consider an online supervised learning problem from a class of model hypotheses under the probably-approximately-correct (PAC) learning framework and propose that decisions for individuals that arrive within $K$ time periods of each other must satisfy individual fairness. They design an algorithm that is asymptotically consistent, i.e., learns and settles on the true hypothesis, with high probability. In contrast, we propose fairness-in-hindsight by only requiring lower bounds on present decisions to satisfy individual fairness relative to {\it all} past decisions, motivated by the notion of conduciveness of decisions. Moreover, our focus is on utility maximization as opposed to pure learning.%Predictions under the true hypothesis are assumed to be individually fair in \cite{heidari-2018}, whereas in our case the optimal unconstrained decisions under the learned utility model may not be individually fair and so at best we can hope to settle on optimal individually fair decisions [i did not understand the last line].} 

%The individual fairness constraint amounts to requiring that across time, two data points that are close to each other must be mapped to similar predictions. Similar to our negative observation in the context of regret, they also observe that if one imposes individual fairness across time, it is impossible to eventually learn and make predictions according to the true hypothesis with high probability. They then relax this definition to only require that individual fairness must be satisfied in any window of $K$ consecutive data-points where $K<\infty$. Under this definition, they design an algorithm that is asymptotically consistent, i.e., learns and settles on the true hypothesis, with high probability. 
%They also characterize the optimal learning rates and their dependence on $K$. 
%There are three main distinctions in our work. First, 

Several other works have recently appeared in this domain. \cite{jabbari2017fairness} extend the notion of fairness in \cite{joseph2016fairness} to the setting of reinforcement learning. \cite{liu2017calibrated} and \cite{gillen2018online} study the model of \cite{joseph2016fairness} under different notions of individual fairness. \cite{liu2017calibrated} require similar probabilities of picking two arms whose quality distribution is similar. They study the issue of calibration under this requirement since the definition is not restrictive enough, e.g., it does not require that these probabilities are monotonic in the average quality.  \cite{gillen2018online} require that similar individuals must face a similar probability of being chosen by the algorithm, except that only a noisy feedback about the distance metric between individuals is available. They study the problem of regret minimization compared to optimal individually fair policy relative to the true metric. \cite{celis2018algorithmic} consider a contextual bandit problem arising in personalization and address the problem of ensuring another notion of fairness called {\it group fairness}\footnote{Group fairness tries to address the issue of {\it disparate impact} in automated decisions: which refers to practices that collectively allocate a more favorable outcome to one population subgroup compared to another. An algorithm satisfies group fairness aka statistical parity if its decisions are independent of membership in any subgroup.} across time. In all of these works, the settings, the models, and the fairness constraints are different from the ones we consider in the present work.

\section{Static and Dynamic Models}\label{sec:model} 
\paragraph{The static model.} Consider a principal responsible for mapping contexts $c \in \mathcal{C}$ to scalar decisions $x \in \cX = [0,1]$, where $\cC$ is a finite set with $|\cC|=C$. We assume that the contexts are drawn from some distribution $\mathcal{D}$ over $\cC$. For a context $c$ and decision $x$, the principal observes a random utility $U$ drawn from some distribution $\cF(x,c)$ over $\mathbb{R}$ and we will often work with the corresponding expected utility $\bar{u}(x,c)\triangleq\mathbb{E}_{\cF(x,c)}(U)$. To achieve meaningful results, we assume that $\bar{u}(x,c)$ is uniformly bounded, i.e., $\max_{c\in \cC, x\in\cX}|\bar{u}(x,c)|\leq B<\infty$. We first consider the case when the distribution $\cF$ is known to the principal; i.e., no learning is required and we will later consider the case where this distribution needs to be learned in Section \ref{sec:dynamic}. 

%Consider a principal responsible for mapping contexts to decisions. Contexts $c$ lie in the set $\cC\subseteq \mathbb{R}^m$ and are drawn from some distribution $\mathcal{D}$ over $\cC$. We assume that $\cC$ is a finite set and $|\cC|=C$. Decisions $x$ are scalar and lie in the set $\cX=[0,1]$ \textcolor{red}{do we need this to be [0,1] in the definitions?}. For a context $c$ and decision $x$, the principal observes a random utility $U$ drawn from some distribution $\cF(x,c)$ defined on some set $\cS\subseteq\mathbb{R}$. For each $c\in\cC$, define  $f(x,c)\triangleq\mathbb{E}_{\cF(x,c)}(U)$. Throughout the rest of the paper, we assume that $f(x,c)$ is uniformly bounded, i.e., $\max_{c\in \cC, x\in\cX}|f(x,c)|\leq B<\infty$. For now, let's assume that the distribution $\cF(x,c)$, for each $c$ and $x$, is known to the principal; we will later consider the case where this distribution needs to be learned. A \emph{decision-rule} is a function $\phi:\cC\rightarrow\cX$ that maps each context $c \in \cC$ to a decision in  $\cX$. 

\begin{example}
{\it Suppose the principal is a bank who is making loan approval decisions. The decision space is $\cX = [0,1]$ representing the amount of loan sanctioned (normalized to 1). The probability of loan default depends on both the loan amount $x$ and the type $c$ of the applicant belonging to the finite set of types $\cC$. Suppose that for a type $c$ and a loan amount $x$, the probability of loan default is estimated to be $p(x,c)$. For a decision $x$, the utility of the bank is $-x$ if there is a default and it is $\beta x$ (the net present value of the interest) if there is no default, i.e.,
$$
U =
\left\{
	\begin{array}{ll}
		-x  & \mbox{w.p. } p(x,c) \\
		\beta x & \mbox{w.p. }1-p(x,c)
	\end{array}
\right. 
$$
Then, the expected utility is $\bar{u}(x,c) = -x p(x,c) + \beta x(1- p(x,c))=x(\beta-p(x,c)(1+\beta))$.}%Thus, in this case, $f(x,c) = \beta-G(x,c)(1+\beta)$.
\end{example}
%\begin{example}
%Suppose the principal is a firm making pricing decisions on a product sold online. Suppose that the prices lie in the set $[0,1]$. Different contexts $c\in\cC$ arrive on different online channels and for each context $c\in\cC$, the probability that a price $x$ is accepted is given by $\bar{F}(x,c)$. Thus for a price $x$ and context $c$, the utility of the bank is $x$ if the object is bought, and $0$ otherwise, i.e., $U= xF$, where, [ADD ELABORATION WITH PROPUBLICA EXAMPLE]
%$$
%F =
%\left\{
%	\begin{array}{ll}
%		1  & \mbox{w.p. } \bar{F}(x,c) \\
%		0 & \mbox{w.p. }1-\bar{F}(x,c)
%	\end{array}
%\right. 
%$$
%Thus the expected utility is $\mathbb{E}(U)=x\bar{F}(x,c)$. In this case, $f(x,c) = \bar{F}(x,c)$.
%\end{example}

Suppose that for any two contexts in $\cC$, there exists a commonly agreed upon distance between them as defined by a 
%\footnote{We acknowledge the difficulty in achieving such a common agreement.} 
function $d_\mathcal{C}: \mathcal{C}\times \mathcal{C}\rightarrow \mathbb{R}^+$. We assume that this function defines a metric on $\cC$; in particular, it is non-negative, satisfies the triangle inequality, and the distance of a context to itself is zero. 
%For the remainder of the paper we assume that $d_{\cX}(x_1,x_2) = |x_1-x_2|$
Consider the following definition of an {\it individually fair} decision-rule in spirit of \cite{Dwork2012}.
\begin{definition}\cite{Dwork2012} A decision-rule $\phi$ is $K$-Lipschitz for $K\in [0,\infty)$ if 
\begin{align}
|\phi(c)- \phi(c')| &\leq K d_{\cC}(c,c') \textup{ for all } c,\,c'\in \cC. 
\end{align}
\end{definition}
Let $\Phi_{K}: \cC \rightarrow \cX$ be the space of $K$-Lipschitz decision-rules that map contexts to decisions. The optimization problem of the principal is to choose a $K$-Lipschitz decision-rule that maximizes the expected utility. We define the maximum expected utility over $K$-Lipschitz decision-rules as: 
\begin{align}
U_{K}\triangleq\max_{\phi\in \Phi_K} \mathbb{E}_{\mathcal{D}}[\bar{u}(\phi(c),c)].\label{opt0}
\end{align}
Note that if the function $\bar{u}(x,c)$ is concave in $x \in [0,1]$ for each $c\in\cC$, then this problem can be solved as a finite convex program since $\cC$ is assumed to be finite, and the maximum is attainable.
\medskip

\paragraph{The dynamic model.} Consider now a discrete time dynamic setting where time is denoted as $t= 1,\cdots,T$ and contexts $c_t\in \cC$ are drawn i.i.d. from the distribution $\mathcal{D}$ over $\cC$ at each time period. The principal makes a decision $x_t \in \cX = [0,1]$, using a {\it policy} $\psi$ that maps the sequence of contexts seen up to time $t$, the corresponding decisions up to time $t-1$, and the utility outcomes up to time $t-1$ to the decision $x_t$ (for all $t\geq1$). Note that a policy is distinct from a decision-rule: a decision-rule is a {\it static} object that maps every possible context to a decision, whereas a policy adaptively maps contexts to decisions as it encounters them, possibly mapping the same context to different decisions across time. %The decisions of the principal, $x_t$ at any time $t$, lie in the set $\cX=[0,1]$ \textcolor{red}{(generalize?)}. 
As in the static case, for context $c_t$ and decision $x_t$, the principal obtains a random utility $U_t$, drawn from the distribution $\cF(x_t,c_t)$ independently of the past. Given the contexts observed and decisions taken, the expected utility of the principal until time $T$ is given by 
$\sum_{t=1}^T\mathbb{E}[U_t] = \sum_{t=1}^T\bar{u}(x_t,c_t)$. %A \emph{policy} for the principal maps the sequence of contexts seen up to time \textcolor{blue}{$t$ (this was t-1??)}, \textcolor{red}{the corresponding decisions up to time $t-1$, and the utility outcomes up to time $t-1$, to a decision $x_t \in \cX$ for all $t\geq1$}. 
We consider the following two definitions of fairness of policies. 
\begin{definition}{\bf(Fairness-across-time)}
We say that a policy is fair-across-time (FT) with respect to the function $\mathcal{K}(s):\mathbb{N}\rightarrow \mathbb{R}^+$ if the decisions it generates for any sequence of contexts satisfy,
\begin{align}
|x_t- x_{t'}| &\leq \mathcal{K}(|t'-t|) d_{\cC}(c_t,c_{t'}) \textup{ for all } t' \neq t. 
\end{align}
When $\mathcal{K}(s) = K$ for some $K\in [0,\infty)$, we say that the policy is $K$-fair-across-time ($K$\textup{-FT}). 
\end{definition} 
Note that by setting $\mathcal{K}(\cdot)$ to be a monotone increasing function, one can model the scenario where the past decisions have a diminishing impact on the future decisions dependent on the amount of time passed. However, even if $\mathcal{K}$ is monotone increasing, fairness-across-time requires that any particular context must be mapped to the same decision irrespective of when it arrives in time (and this hinders learnability as we discuss further in Section \ref{sec:dynamic}). %\textcolor{blue}{(is it only the case when a context appears more than once?)} add a note about future 
%The following notion is weaker than fairness-across-time.
\begin{definition}{\bf(Fairness-in-hindsight)}
We say that a policy is fair-in-hindsight (FH) with respect to the function $\mathcal{K}(s):\mathbb{N}\rightarrow \mathbb{R}^+$ if the decisions it generates for any sequence of contexts satisfy,
\vspace{-0.2cm}
\begin{align}
x_t\geq x_{t'} -\mathcal{K}(t-t') d_{\cC}(c_t,c_{t'}) \textup{ for all } t\geq t'.
\end{align}
When $\mathcal{K}(s) = K$ for some $K\in [0,\infty)$, we say that the policy is $K$-fair-in-hindsight ($K$\textup{-FH}).
\end{definition}
Note that fair-in-hindsight policies must make monotone (non-decreasing) decisions over time for any given context. In fact, setting $\mathcal{K}(s) = 0$ for all $s$, we recover policies that make monotone decisions over time irrespective of the context.
%\begin{example} If a person A gets approved of a loan amount that is significantly higher than that of another similar person B, then this wouldn't be perceived unfair to A. If B applied earlier than A, then our FIH condition says that this wouldn't be unfair to B, as long as B got approved of a loan that is (approximately) at least as much as that approved for similar people who applied before her. In other words, the approved loan amount couldn't be deemed individually unfair to B based on the history of decisions when it was approved -- it could later turn out to be individually unfair in retrospect, but FIH allows for that. TF on the other hand requires individual fairness relative to the future as well as past decisions.\end{example}
%For the remainder of the paper we will focus our attention on $K$-FT and $K$-FH policies only. 
Let $\Psi^T_{K\textup{-FT}}$ and $\Psi^T_{K\textup{-FH}}$ be the space of $T$-horizon policies that are $K$-FT and $K$-FH respectively. Let the maximum attainable expected utility up to time $T$ using $K$-FT and $K$-FH policies be $U^{T}_{K\textup{-FT}}$ and $U^{T}_{K\textup{-FH}}$: \begin{align}
U^{T}_{K\textup{-FT}} &:= \max_{\psi\in \Psi^T_{K\textup{-FT}}}\sum_{t=1}^T\mathbb{E}_{\mathcal{D}}[\bar{u}(x_t,c_t)],& U^{T}_{K\textup{-FH}} &:= \max_{\psi\in \Psi^T_{K\textup{-FH}}} \sum_{t=1}^T\mathbb{E}_{\mathcal{D}}[\bar{u}(x_t,c_t)].
\end{align}
%The expectations are with respect to the randomness in the sequence $(c_t)_{t\geq 1}$. Define the optimal values of these problems as $U^{T}_{K\textup{-FT}}$ and $U^{T}_{K\textup{-FH}}$, respectively. 
It is clear that both $U^{T}_{K\textup{-FT}} \geq T U_{K}$ and $U^{T}_{K\textup{-FH}} \geq T U_{K}$, since one can simply use the optimal $K$-Lipschitz decision-rule at every stage. But for small horizons, one can potentially do better. Intuitively, this is because you may not expect to encounter all the contexts within a short horizon; hence the fairness constraints are expected to be less constraining, thus offering more flexibility in mapping contexts to decisions. For the interested reader, we present an example below to show that one can attain a higher utility under FH than under FT, and both these notions can lead to a higher utility that following the static optimal $K$-Lipschitz decision-rule when the horizon is small.

\begin{example}\label{ex:twostage} {\it Let $T=2$. Suppose there are two contexts: $A$, $B$ where $A$ is seen with probability 1/12 and $B$ is seen with probability 11/12. Let the expected utility be $\bar{u}(A, x) = - x$ (defaulters) and $\bar{u}(B,x) = 1.2x$ (non-defaulters) for decisions $x \in [0,1]$. Note that optimum unconstrained decision for $A$ is 0 and optimum unconstrained decision for $B$ is 1. Let $d(A, B) = 1$ and $K=0.5$. In the optimal $K$-Lipschitz decision rule, the decision for $A$ is $0.5$ and the decision for $B$ is $1$.
  
Now any $K$-FT policy $\psi$ must give loans $x_1$ to $A$ and $x_2$ to $B$ such that $|x_1-x_2| \leq 0.5$ irrespective of the time periods they arrive in. Any $K$-FH policy $\psi^\prime$ must ensure that if loan $x_1$ was given to context $c_1$ at $t=1$, then at least the same amount of loan must be given to the same context, and at least $x_1 - 0.5$ must be given to the other context at $t=2$, dependent on which context arrives. Suppose now that at $t=1$, $A$ arrives. Then, one can verify that the optimal $K$-FT policy must give a loan of 0.5 in anticipation of $B$ in the next time step (so that a loan of 1 can be given to $B$), whereas the optimal $K$-FH policy can give a loan of 0 to $A$ at $t=1$. If A arrives at $t=2$, then $K$-FH can still give a loan of 0; if $B$ arrives then it can give a loan of 1. Thus one can attain a higher utility under FH than under FT.

To see that both the notions FT and FH can lead to a higher utility than that under the static optimal $K$-Lipschitz decision-rule, observe that if there is only a single stage, i.e., $T=1$, then the decision for $A$ can be $0$ under any FT or FH policy, whereas the optimal $K$-Lipschitz decision-rule must choose $0.5$.}
\end{example}

We can show, however, that when the horizon gets longer, one cannot do any better than achieving the static optimum average expected utility $U_{K}$ (as defined in \eqref{opt0}). 

%Let $\Psi^T$ \textcolor{red}{(this is never used, I think)} be the space of all policies for a fixed horizon $T$, and let $\Psi^T_{K\textup{-FT}}$ and $\Psi^T_{K\textup{-FH}}$ be the space of $T$-horizon policies that are $K$-FT and $K$-FH respectively. Consider the following two optimization problems for the principal.
%\begin{align}
%P^{T}_{K\textup{-FT}}:\,\,&\max_{\psi\in \Psi^T_{K\textup{-FT}}}\frac{1}{T}\sum_{t=1}^T\mathbb{E}[u(x_t,c_t)] \hspace{0.1in}\\
%\,\,&\textrm{and}\,\,\nonumber\\
%P^{T}_{K\textup{-FH}}: \,\,&\max_{\psi\in \Psi^T_{K\textup{-FH}}} \frac{1}{T}\sum_{t=1}^T\mathbb{E}[u(x_t,c_t)].
%\end{align}
%The expectations are with respect to the randomness in the sequence $(c_t)_{t\geq 1}$. Define the optimal values of these problems as $U^{T}_{K\textup{-FT}}$ and $U^{T}_{K\textup{-FH}}$, respectively. It is clear that both $U^{T}_{K\textup{-FT}} \geq U_{K}$ and $U^{T}_{K\textup{-FH}} \geq U_{K}$, since one can simply use the optimal $K$-Lipschitz decision-rule at every stage. But for small horizons, \textcolor{red}{one can potentially do better}. Intuitively, this is because you may not expect to encounter all the contexts within a short horizon; hence the fairness constraints are expected to be less constraining, thus offering more flexibility in mapping contexts to decisions. But we can show that when the horizon gets longer, one can't do any better than achieving $U_{K}$ on average as defined in \eqref{opt0}.

\begin{proposition} \label{prop:benchmark}
For any $K\in [0,\infty)$, we have $U^{T}_{K\textup{-FT}}\leq TU_K+2B$ and $U^{T}_{K\textup{-FH}}\leq TU_K+2B$ where $B$ is the upper bound on the possible expected utility, i.e. $\max_{c\in C,x\in X} |\bar{u}(x, c)| \leq B < \infty$. Hence, 
$$\lim_{T\rightarrow\infty}\frac{1}{T}U^{T}_{K\textup{-FT}} =\lim_{T\rightarrow\infty}\frac{1}{T}U^{T}_{K\textup{-FH}} = U_{K}.$$
\end{proposition}
%\textcolor{red}{Changed old lower bound $u_t$ to $\ell_t$, and changed old utility $f(x,c)$ to $u(x,c)$.}

%For the K-FT policy, the proof is straightforward: after a random $T'$ time periods, every context $c\in \cC$ will have been realized at least once, where $\textup{E}(T') = \textup{O}(1)$. After this time, a $K$-Lipschitz decision-rule gets fixed since every context should be matched to the same decision every time. Hence the expected average long-run utility cannot be higher than $U_{K}$. Under the K-FH policy, we can show that once every context $c\in \cC$ has realized at least once, the decision-rule in each time period from that point on must be $K$-Lipschitz and hence the expected per period utility cannot be higher than $U_{K}$. 

\begin{proof}{Proof.}
%[Proof of Proposition~\ref{prop:benchmark}]
%Let $T'$ be the random time at which every context $c\in\cC$ has been realized at least once. Since $\cC$ is finite, $\mathbb{E}(T')= C\leq \infty$ where $C$ depends only on the distribution over contexts $\mathcal{D}$. After this time, under the fairness-across-time (FT) constraint, a $K$-Lipschitz decision-rule gets fixed since every context should be matched to the same decision every time. Hence, after this time, the per period expected utility cannot be larger than $U_K$, and thus the expected average long-run utility cannot be higher than $U_{K}$. 
Note that $U^{T}_{K\textup{-FT}}\leq U^{T}_{K\textup{-FH}}$ since FT implies FH. Hence, we show the result only for $U^{T}_{K\textup{-FH}}$. The corresponding result for $U^{T}_{K\textup{-FT}}$ follows. 

Fix an FH policy. At any given time $t$, let $\ell_t(c)$ be the tightest lower bound on the decision for $c$, for each $c\in\cC$, based on decisions taken in the past. Note that for any decision $x$ taken for a context $c$ in the past, $\ell_t(c) \geq x$. 

First, we show that $\ell_t(\cdot)$ specifies a $K$-Lipschitz decision-rule. To see this, consider two contexts $c$ and $c'$ and w.l.o.g., assume that $\ell_t(c)\geq \ell_t(c')$. If $\ell_t(c) = 0$, then clearly $\ell_t(c)=\ell_t(c')=0$. Next, if for some time $t'<t$, the context $c$ was mapped to decision $\ell_t(c)$, then from the FH constraint, it follows that $\ell_t(c')\geq \ell_t(c) - Kd_\cC(c,c')$. Thus $|\ell_t(c)-\ell_t(c')|\leq Kd_\cC(c,c')$. Finally, suppose that either the context $c$ had never appeared before time $t$, or it had appeared and the highest decision taken for this context so far is some $x<\ell_t(c)$ (note again that the highest decision in the past for context $c$ cannot be larger than $\ell_t(c)$). In this case, there is some other context $c^*$ that was mapped to some decision $x^*$ at some time in the past and $\ell_t(c) = x^*-Kd_\cC(c^*,c)$ (since $\ell_t(c)$ is the tightest lower bound). But this also means that $\ell_t(c')\geq x^*-Kd_\cC(c^*,c')$. Thus $\ell_t(c')-\ell_t(c)\geq K(d_\cC(c^*,c)-d_\cC(c^*,c'))$. But by the triangle inequality, we have $d_\cC(c^*,c')\leq d_\cC(c^*,c) + d_\cC(c,c')$. Thus we have $\ell_t(c')\geq \ell_t(c)-Kd_\cC(c,c')$. Thus again, $|\ell_t(c)-\ell_t(c')|\leq Kd_\cC(c,c')$. This shows that $\ell_t(\cdot)$ is $K$-Lipschitz.

%Suppose for a contradiction this is not the case. 
Now consider the decision rule $\phi_t$ chosen by the policy at time $t$. Our overall proof strategy is as follows. We will bound from above the expected utility under $\phi_t$ at time $t$ by the expected utility of the decision rule $\ell_t(.)$ plus a side-payment. Since $\ell_t(.)$ is a $K$-Lipschitz decision-rule, the expected utility under this decision-rule is at most $U_K$. Additionally, we will show that over time, the total side-payments are bounded by a constant independent of $T$. 

To see this, suppose that $\phi_t$ is replaced by $\ell_t(.)$. Now any change in expected utility due to this switch can be compensated by a side-payment of $\bar{u}(\phi_t(c),c)-\bar{u}(\ell_t(c),c)$ to the principal in the event that $c$ arrives at time $t$, for each $c\in\cC$, i.e., by an expected side payment of $\mathbb{E}_\cD[\bar{u}(\phi_t(c),c)-\bar{u}(\ell_t(c),c)]$. Moreover, if $c$ arrives at time $t$, then at time $t+1$, $\ell_{t+1}(c) =  \phi_t(c)$. Thus, the total expected side-payment over all arrivals, irrespective of $T$, is at most $\mathbb{E}_\cD[\max_{x\in\cX}|\bar{u}(x,c)-\bar{u}(0,c)|]\leq 2B$. Thus we have an upper bound on the expected utility under any policy, equal to $TU_K+2B$.

\hfill $\square$
\end{proof}

This, in particular, shows that relaxing the fairness-across-time constraint to only requiring fairness-in-hindsight does not lead to any long-run gains in objective. The policy of simply choosing the optimal static $K$-Lipschitz decision-rule at every stage is approximately optimal for a large horizon $T$. We show next that the situation is drastically different when there is learning involved.
%\begin{align}
%\max_{\phi \sum_{t=1}^T\bx_t.\bof(c_t)
%\end{align}
\medskip

\section{Dynamic Model with Learning}\label{sec:dynamic} Consider now a setting where the distribution of the utility given a context and a decision is unknown to the principal and must be learned. Formally, this distribution depends on an additional unknown parameter $w$, which we assume to belong to a finite set $\cW$. With some abuse of notation, for each $w\in\cW$,  $c\in\cC$ and $x\in\cX$, the distribution of the utility of the principal is given by $\cF(x,c,w)$. We assume that this distribution has a finite support $\cU(x,c)$, i.e. $\max_{x,c}|\cU(x,c)|<\infty$,\footnote{This is satisfied in Example 4.1 where there are only two possibilities: either the person defaults on the loan or doesn't.}  and for each $u\in\cU(x,c)$, the probability of observing $u$ is given by $p(u\mid x,c,w)$. We assume that $p(u\mid x,c,w)>0$, for each $u\in\cU(x,c)$ for all $c\in\cC$, $w\in\cW$, and for all $x$ in the interior of $\cX$, i.e., $x\in(0,1)$. Finally, we assume that the principal knows possible feasible parameters $\cW$ but does not know the true parameter $w$, which must be learned by adaptively assigning decisions to contexts and observing the outcomes. 

%Let $x_w$ be the lowest decision one needs to take across contexts given that the state of the world is $w$, in the optimal IF policy. 
%Let $D_{\textup{KL}}(\cF\| \cG)$ denote the Kullback-Liebler (K-L) divergence between any two distributions $\cF$ and $\cG$ defined on the same finite support such that $\cF$ is absolutely continuous relative to $\cG$.\footnote{\vj{Absolute continuity means that if $\cG$ assigns a $0$ probability to an element of the support, $\cF$ assigns a $0$ probability to that element as well.}} In our setting, 
%By assumption 1, $LR(x)>0$ for all $x\in (0,x^*]$.
%\begin{assumption}[Learnability at low decisions]
%$$\inf\{x\in\cX;\,\min_{w,w'}\mathbb{E}_{\cD}[KL(w,w'|x,c)]>0\} = 0.$$
%\end{assumption}

%\textcolor{red}{notation for KL, script G.}
We now redefine some previously defined quantities (with some abuse of notation) to capture the dependence on the parameter $w$. 
First, we define $\bar{u}(x,c,w)\triangleq\mathbb{E}_{\cF(x,c,w)}(U)$, i.e., the mean utility for a given $x\in\cX$, $c\in\cC$ and $w\in\cW$. 
%\textcolor{red}{Is this \bar{u}(x,c,w) related to $u \in \mathcal{U}$ in the previous section?}
As before, we assume that $\bar{u}(x,c,w)$ is uniformly bounded, i.e., $\max_{c\in \cC,\, x\in\cX,\, w\in\cW}|\bar{u}(x,c,w)|\leq B<\infty$. 
%For the results in this section, we further assume that $\bar{u}(\cdot,c,w)$ is Lipschitz continuous on $\cX$ for each $c\in\cC$ and $w\in\cW$, with a Lipschitz constant $R>0$, i.e.,
%$$|\bar{u}(x,c,w)-\bar{u}(x',c,w)|\leq R|x-x'|,$$
%for all $x,\,x'\in\cX$, $c\in\cC$ and $w\in\cW$. 
Moreover, we assume that $\bar{u}(x,c,w)$ is continuous at all $x\in\cX$ for each $c\in\cC$ and $w\in\cW$. Next, we define $U_{K}(w)$ to be the highest expected utility attainable under a $K$-Lipschitz decision rule, for a given parameter $w$, i.e.,
\begin{align}
U_{K}(w)\triangleq\max_{\phi\in \Phi_K} \mathbb{E}_{\mathcal{D}}[\bar{u}(\phi(c),c,w)].\label{opt0}
\end{align}
Let $\phi^*_w$ denote the optimal $K$-Lipschitz decision rule that attains this maximum. 

For a given horizon $T$, for any dynamic policy in $\Psi^T_{K\textup{-FT}}$ or $\Psi^T_{K\textup{-FH}}$ that is oblivious of $w$, we can define a notion of {\it regret} that compares its expected utility against the long-run optimal benchmark $TU_{K}(w)$. For any policy $\psi\in \Psi^T_{K\textup{-FT}}$, for a fixed $w$, we denote its total utility at the end of the horizon as $U^T_{K\textup{-FT}}(w, \psi)$ (similarly, $U^T_{K\textup{-FH}}(w, \psi)$). Then, for any $\psi\in \Psi^T_{K\textup{-FT}}$, let the regret be denoted as:
\begin{align}
\Reg^T_{K\textup{-FT}}(w,\psi) \triangleq TU_{K}(w)-U^T_{K\textup{-FT}}(w, \psi),
\end{align}
and similarly, for any $\psi\in \Psi^T_{K\textup{-FH}}$, we define,
\begin{align}
\Reg^T_{K\textup{-FH}}(w,\psi) \triangleq TU_{K}(w)-U^T_{K\textup{-FH}}(w, \psi).
\end{align}

\subsection{Learning under FT}
We show that under the FT constraint, except for trivial settings, a regret that asymptotically grows linearly in $T$ is unavoidable. The reason is that once a context is mapped to a decision, we are forced to map that context to the same decision forever under FT. We can show that with some positive probability, a bad decision in the first step for some $w$ is inevitable, which then must be repeated forever, thus incurring linear regret. This intuition is illustrated in Example 1.1 in Section~\ref{sec:intro}. 
%\begin{example}
%Suppose the context for each loan applicant simply denotes whether they are aged below 45 or aged above 45, and the bank does not know whether age is positively or negatively correlated with default probability. If the first applicant is aged above 45 and is given a loan of amount \$M, then  any future applicant aged above 45 must be given \$M to satisfy fairness-across-time. But 
%this decision of \$M loan is bound to be suboptimal and leads to a linear regret when \$M is small but age is negatively correlated with default probability or when \$M is large and age is positive correlated with default probability. \textcolor{red}{this seems repetitive considering we have a similar example in the introduction}
%this amount of \$M is bound to be incorrect in at least one of the two cases: if the amount is low, then this is bad in the situation where age is negatively correlated with default probability, and if the amount is high, then this is bad in the case where age is negatively correlated with default probability. 
%\end{example}
\begin{proposition}
Suppose there is some pair $w',\,w''\in\cW$ such that a) $\phi^*_{w'}$ and $\phi^*_{w''}$ are the unique optimal (static) $K$-Lipschitz decision-rules for $w'$ and $w''$ respectively,\footnote{Note that the set of $K$-Lipschitz decision rules given a finite set of contexts is convex. By adding a small random perturbation to the expected utility function, we can ensure that the optimal decision rules are unique for each $w \in \cW$.} and  b) $\phi^*_{w'}\neq\phi^*_{w''}$. Then there exists an instance dependent constant $\kappa>0$ and a time $T'\geq 1$ such that for any $T\geq T'$ and any $K\textup{-FT}$ policy $\psi\in \Psi^T_{K\textup{-FT}}$, 
$$\max_{w\in\cW}\Reg^T_{K\textup{-FT}}(w, \psi) \geq \kappa T.$$ 
\end{proposition}
\begin{proof}{Proof.}
Fix some $T\geq 1$ and consider a $K\textup{-FT}$ policy $\psi\in \Psi^T_{K\textup{-FT}}$. Since $\phi^*_{w'}\neq\phi^*_{w''}$, there exists some $c'\in\cC$ such that $\phi^*_{w'}(c')\neq\phi^*_{w''}(c')$. Let $\phi^1$ be the decision-rule followed by policy $\psi$ at time $1$, and let $\phi^1(c')=x'$.  Consider a $K\textup{-FT}$ policy $\psi^{\textup{o}}$, which knows $w$ and maximizes the $T$-period utility for each $w$, except that it is constrained to use the decision rule $\phi^1$ at time $1$ irrespective of $w$, i.e., $x_1 =\phi^1(c_1)$ under $\psi^{\textup{o}}$. It is then clear that for each $w\in\cW$, 
\begin{equation}
\Reg^T_{K\textup{-FT}}(w, \psi)\geq \Reg^T_{K\textup{-FT}}(w, \psi^{\textup{o}}).\label{eqn:oracle}
\end{equation}
Consider now the event that the context seen at time 1 is $c^\prime$, i.e., event $\{c_1=c'\}$ happens. On this event, using a similar argument as in the proof of Proposition \ref{prop:benchmark}, we can show that the policy $\psi^{\textup{o}}$ can achieve at most $\delta_T = 2B/T$ more than $U'(w,x')$ on an average, where
 $$U'(w,x')= \max_{\phi\in \Phi_K;\, \phi(c') = x'} \mathbb{E}_{\mathcal{D}}[\bar{u}(\phi(c),c,w)].$$
Thus we have, \begin{align*}
 \max_{w\in\{w',w''\}}\Reg^T_{K\textup{-FT}}(w, \psi^{\textup{o}})&\geq \mathbb{P}(c_1=c')\Big[\max_{w\in\{w',w''\}}T\Big(U_K(w)-U'(w,x')-\delta_T\Big)\Big]\\
&~=\mathbb{P}(c_1=c')\Big[\max_{w\in\{w',w''\}}T\Big(U_K(w)-U'(w,x')\Big)-2B\Big].
\end{align*}
Clearly, $U_K(w)\geq U'(w,x')$ for any $w$, and hence $ \max_{w\in\{w',w''\}} (U_K(w)-U'(w,x'))\geq 0$. If $\max_{w\in\{w',w''\}} (U_K(w)-U'(w,x')) = 0$, then this implies that 
$U_K(w) = U'(w,x')$ for both $w=w'$ and $w=w''$. But this implies the existence of $K$-Lipschitz decision-rules for $w'$ and $w''$ that are optimal, and such that they both map $c'$ to $x'$. This contradicts the fact that $\phi^*_{w'}$ and $\phi^*_{w''}$ are the unique optimal $K$-Lipschitz decision-rules, since they map $c'$ to different decisions. Thus $\max_{w\in\{w',w''\}} U_K(w)-U'(w,x') > 0$ for each $x'\in\cX$. This in turn implies that $\min_{x'\in\cX} \max_{w\in\{w',w''\}} U_K(w)-U'(w,x') \triangleq \gamma'> 0$ (the minimum is well-defined since $U'(w,x')$ is a continuous function of $x'$ for each $w$).

Hence, $\max_{w\in\{w',w''\}}\Reg^T_{K\textup{-FT}}(w, \psi^{\textup{o}}) \geq \mathbb{P}(c_1=c')[\gamma' T-2B]$. Using \eqref{eqn:oracle}, we get the result. \hfill $\square$
\end{proof}
%In this case, it is easy to construct an example where for any sequence of $K\textup{-FT}$ policies $(\psi_T)_{T\in\mathbb{N}}$, $\Reg^T_{K\textup{-FT}}(w, \psi)$ grows linearly in $T$. 
This shows that fairness-across-time can result in significant losses in utility relative to the static optimum when learning is involved.

\subsection{Learning under FH}
The situation is not as bleak under the FH constraint as we now show. The problem with the FT constraint is that one is forced to repeat mistakes by mapping each context to the same decision that was made when the context was first encountered. The FH constraint allows for some flexibility: the decisions for every context can potentially {\it increase} over time. This presents the following possibility: one can try to {\it cautiously} learn $w$ with small and equal decisions for all contexts, and then increase decisions for contexts as needed once $w$ is learned with an appropriate confidence. This was illustrated in Example 1.2 in Section~\ref{sec:intro}. 
%\begin{example}
%Consider Example 2, where the bank is initially unaware whether age is positively or negatively correlated with default probability. In this case, the bank can approve a small amount of loan, say $\epsilon$, to each applicant in an initial ``explore" phase. Once the bank learns the correlation structure with appropriate confidence, it can start approving loans for larger amounts as appropriate in an individually fair manner, while guaranteeing a loan of $\epsilon$ to everyone. This ensures that fairness holds in the sense of FH. 
%and a suitable choice of $\epsilon$ controls the regret as a function of the decision-making horizon. % If the first applicant is above age 45, whatever amount of loan is allocated to this person must be allocated to any future applicant above age 45 to satisfy temporal fairness. But this amount is bound to be incorrect in at least one of the two cases: if the amount is low, then this is bad in the situation where age is negatively correlated with default probability, and if the amount is high, then this is bad in the case where age is negatively correlated with default probability. 
%\end{example}
Formally, we propose an online learning algorithm that we call {\it Cautious Fair Exploration} or {\sc CaFE} in Algorithm 1. 

In order to describe our performance bound for \CFE, we first need to define a few quantities. For any $w,\,w'\in\cW$, $x\in (0,1)$ and $c\in\cC$, we define 
\begin{align}
\KL(w,w'|x,c)&\triangleq \sum_{u\in\cU(x,c)}p(u\mid x,c,w)\log \frac{p(u\mid x,c,w)}{p(u\mid x,c,w')}.
 \end{align}
This quantity is commonly known as the Kullback-Liebler (K-L) divergence between the distributions $\cF(x,c,w)$ and $\cF(x,c,w')$. 
%For notational convenience we denote $D_{\KL}(\cF(x,c,w)\| \cF(x,c,w'))$ as $\KL(w,w'|x,c)$. 
 %\textcolor{red}{define the following quantities to denote the expectation over all contexts and the minimum with respect to any pair of parameters:}
%$$L(x,w) \triangleq \min_{w'}\mathbb{E}_{\cD}[\KL(w,w'|x,c)], \textrm{ and}$$
Also, recall that $c \sim \cD$ and for each $x\in(0,1)$, define:
\begin{align}
L(x) \triangleq \min_{w,w'}\mathbb{E}_{\cD}[\KL(w,w'|x,c)].\label{eq:Lx}
\end{align}
$L(x)$ captures how well we can distinguish between the different parameter values given a decision $x$: a low value of $L(x)$ implies that there is a pair $(w,\, w')$ that is difficult to distinguish at decision $x$. We also define for each $x\in\cX$: 
\begin{align}
D(x) \triangleq\min_{c\in\cC,\,w\in\cW,\,u\in\cU(x,c)}p(u\mid x,c,w).\label{eq:Dx}
\end{align}
$D(x)$ captures the smallest probability assigned to a utility value over all possibilities for $w$ and $c$, as a function of the decision $x$. Finally, we make the following assumption.
\begin{assumption}\label{as:1}
$\bar{u}(x,c,w)$ is Lipschitz continuous on $\cX$ for each $c\in\cC$ and $w\in\cW$, with a Lipschitz constant $R>0$, i.e.,
$$|\bar{u}(x,c,w)-\bar{u}(x',c,w)|\leq R|x-x'|,$$
for all $x,\,x'\in\cX$, $c\in\cC$, and $w\in\cW$. 
\end{assumption}

%\begin{assumption}\label{as:1}
%$\bar{u}(x,c,w)$ is Lipschitz continuous on $\cX$ for each $c\in\cC$ and $w\in\cW$, with a Lipschitz constant $R>0$, i.e.,
%$$|\bar{u}(x,c,w)-\bar{u}(x',c,w)|\leq R|x-x'|,$$
%for all $x,\,x'\in\cX$, $c\in\cC$, and $w\in\cW$. 
%\end{assumption}
%\begin{assumption}\label{as:1}
%$\displaystyle\min_{u\in\cU,\,c\in\cC,\,w\in\cW}p(u\mid x,c,w) =\Omega(x^H)$ for some $H\geq 0$ as $x\rightarrow 0$.
%\end{assumption}
\begin{tcolorbox}[float, title =Algorithm 1: Cautious Fair Exploration (\CFE) ]
%\textbf{Algorithm 1: Cautious Fair Exploration (\CFE)}\\
 \textbf{Input:} $T\in \mathbb{N}$, $(\cF(x,c,w); \, c\in\cC,\, x\in\cX,\, w\in\cW)$, $\epsilon \in [0,1]$, Lipschitz constant $K$.\\
\textbf{Definitions:}
For $1\leq t\leq T$, let $\lambda_t(w)$ be the likelihood of $w\in\cW$ based on observations until time $t$, i.e., 
$$\lambda_t(w) = \prod_{i=1}^t p(U_i\mid x_i,c_i,w).$$
For each $w,w'\in\cW$, define $\Lambda_0(w,w')=0$ and $\Lambda_t(w,w')=\log \frac{\lambda_t(w)}{\lambda_t(w')}$ for $1\leq t\leq T$. \\
%Let $w^*_t =\arg\max_{w\in\cW}l_t(w)$ be the maximum likelihood estimate of the model parameter at time t. \\
\textbf{Policy:} While $1\leq t\leq T$ do:\\
\begin{enumerate}
\item \textbf{{\sc Explore}:} While there is no $w$ such that for every $w'\neq w$,
$\displaystyle\max_{s< t}\Lambda_s(w,w') >\log T$, assign $x_t = \epsilon$. Note that there can only be one such $w$.\\
\item \textbf{{\sc Exit from Explore}:}\\ If there is a $w$ such that for every $w'\neq w$,
$\displaystyle\max_{s< t}\Lambda_s(w,w') >\log T$, define $w^* \triangleq  w$ and permanently enter the Exploit phase. \\
\item \textbf{{\sc Exploit}:} Use the static optimal decision rule in $\cX_{\epsilon} = [\epsilon, 1]$ assuming the model parameter is $w^*$, i.e., use the decision rule that solves:
\begin{equation}
\max_{\phi:\cC\rightarrow\cX_{\epsilon}}\mathbb{E}_{\cD}[\bar{u}(\phi(c),c,w^*)]\label{opt2}
\end{equation}
$$\textrm{s.t. } |\phi(c)- \phi(c')| \leq K d_{\cC}(c,c') \textup{ for all } c,\,c'\in \cC. $$
\end{enumerate}
\end{tcolorbox}

%Our performance bound for \CFE~depends on $H$ (although it doesn't operate under the knowledge of these quantities). Following is our main result characterizing this bound.
We present the following main result that characterizes the performance of \CFE.
\begin{theorem}\label{thm:main} Fix a $K\in [0,\infty)$. 
\begin{enumerate}
\item \CFE~is a $K$-\textup{FH} policy.
\item Suppose Assumption~\ref{as:1} holds. Also, suppose that $L(x) =  \Omega(x^M)$ (defined in \eqref{eq:Lx}) and $D(x) =  \Omega(x^H)$ (defined in \eqref{eq:Dx}) for some $M\geq 0$ and $H\geq 0$ as $x\rightarrow 0$, then,
\begin{align*}
&\Reg^T_{K\textup{-FH}}(w, \textup{\CFE}_T)\leq  \beta T^{\frac{M}{M+1}}\log T(1+\textup{o}(1)),
%&~\leq  \frac{R(|\cW| -1)(1 +\frac{H}{M+1})}{a}T^{\frac{M}{M+1}}\log T(1+\textup{o}(1)),
\end{align*}
as $T\rightarrow\infty$, where  $\textup{\CFE}_T$ is \CFE~initialized with $\epsilon=1/T^{\frac{1}{M+1}}$, and $\beta>0$ is a constant that depends on $L(.)$,  $|\cW|$, R, M, and H. In particular, if $M=0$, then \CFE~initialized with $\epsilon=1/T$ attains a regret of $\textup{O}(\log T)$.
%and some constant $a>0$ that satisfies 
%$L(x)> ax^M$ for any small enough $x$.\\
%\item If $\liminf_{x\rightarrow 0} L(x) = S>0,$ then as $T\rightarrow\infty$,
%\begin{align*}
%&\Reg^T_{K\textup{-FH}}(w, \textup{\CFE}_T)\leq C_b\log T (1+\textup{o}(1)),
%&~\leq \frac{R(|\cW| -1)(1+H)}{S}\log T (1+\textup{o}(1)),
%&~
%\end{align*} 
%where $\textup{\CFE}_T$ is \CFE~initialized with $\epsilon=1/T$, and $C_b>0$ is some constant that depends on $|\cW|$, R, S, and H.
%\end{enumerate}
\end{enumerate}
%\textcolor{red}{K is not the same in K-FH, $x^K$ right?}
%\begin{align}
%\max_{w\in\cW}\bigg( U_{K}(w)-U^T_{K\textup{-FT}}(w, \psi^T_\epsilon)\bigg)\leq \epsilon|\cC|\max_{x\in\cX;\,c\in\cC;\,w'\in\cW}|g(x,c,w')|.
%\end{align} 
\end{theorem}

The intuition behind these upper bounds are as follows. In the Explore phase, the parameter $w$ is learned with a probability of error $1/T$ by choosing $x_t = \epsilon$ irrespective of $c_t$. 
%We can show that $\mathbb{E}(T') = \textup{O}(\log T)$ and hence the expected loss in utility relative to $U_{K}(w)$ in this learning phase is $\textup{o}(1)$. 
From that point onwards, the policy enters the exploitation phase: it simply assumes the learned $w^*$ to be the truth and chooses a $K$-Lipschitz decision rule defined on the decision space $\cX_{\epsilon} = [\epsilon,1]$. This defines a $K$-FH policy. Since the decisions are lower bounded by $\epsilon$ for every context during the exploitation phase, one incurs a loss of $\textup{O}(\epsilon)$ per step. Thus we want to choose $\epsilon$ to be as small as possible; in fact, ideally, we would want to pick $\epsilon=0$. But choosing a small $\epsilon$ may force us to learn $w$ prohibitively slowly, thus increasing the length of the Explore phase and hence the overall regret. For example, consider again the setting where the bank is unsure whether age is positively or negatively correlated with loan default rates. It could be the case that overall default rates, irrespective of the context, are small when the loan amounts are small, making it difficult to learn the correlation structure with small loans.

If $L(x) =\Omega(1)$ (i.e., $M=0$) as $x\rightarrow 0$ then we can indeed learn at an adequate rate by picking $\epsilon_T=1/T$, and in this case, we can achieve an overall regret of $\textup{O}(\log T)$. On the other hand, if $L(x) = \Theta(x^M)$ for $M>0$, then $\epsilon_T$ cannot be chosen to be too small so that the Explore phase is not prohibitively long; in this case, the choice of $\epsilon_T = 1/T^{1/(M+1)}$ optimizes the tradeoff between regret incurred during the Explore and Exploit phases, leading to an overall loss of $\tilde{\textup{O}}(T^{\frac{M}{M+1}})$.

%Finally, note that some assumption on learnability at small $x$ is necessary to be able to achieve sublinear regret. To see this, suppose that $\inf\{x\in\cX;\,L(x)>0\} =\delta > 0$. Then there exists some pair $w,\, w'$ such that it is impossible to distinguish them for any $x<\delta$. If the optimal individually fair decision-rules differ under $w$ and $w'$, then $w$ must be distinguished from $w'$, which is only possible if the principal chooses a decision $x>\delta$. But this induces a per

\begin{proof}{Proof of Theorem~\ref{thm:main}.}
First, we show that \CFE~is $K$-FH. To see this, note that the policy is fixed irrespective of the context in the learning phase and hence FH. In the exploitation phase it is FH with respect to any time in the exploitation phase since the exploitation phase uses a $K$-Lipschitz decision rule. Finally, it is also FH with respect to the Explore phase in the Exploit phase since decisions for each context only increase in going from Explore to Exploit. 

Next, we will establish an upper bound on the regret for a fixed model parameter $w$. For a fixed $T$, define $T'$ to be the minimum of $T$ and the (random) time at which the Explore phase ends, i.e.,
\begin{align}
&T'= \max\{1\leq t\leq T\mid\,  \nexists\, w'\in \cW, \,  \max_{s< t}\Lambda_s(w',w'') \geq \log T \,\,\forall\,\, w''\neq w' \}.
\end{align}
Note that if the Explore phase doesn't end before time $T$, then $T'=T$.
Next, consider a new coupled stochastic process $(\bar\Lambda_t(w',w''))_{t\in\mathbb{N}}$ for each $w',w''\in \cW$ that is identical to $(\Lambda_t(w',w''))_{t\leq T}$ upto time $T'$ and then it continues as if the Explore phase did not end (i.e., the allocations are $\epsilon_T$ for all contexts forever). Define $T^{w'}_w$ to be the minimum of T and the random time at which $w$ gets ``distinguished'' from $w'$ in this new stochastic process, i.e., when the log-likelihood ratio of $w$ relative to $w'$ based on the observations crosses the threshold of $\log T$. Formally, 
$$T^{w'}_{w} = \max\{1\leq t\leq T\mid \, \max_{s< t} \bar\Lambda_{s}(w,w') <\log T \}.$$
Define $T_w = \max_{w'\neq w} T^{w'}_{w}$, i.e., $T_w$ is the minimum of T and the random time at which $w$ gets ``distinguished'' from {\it all} $w'\neq w$. 
%$$T_{w}= \max\{1\leq t\leq T\mid\,  \exists\, w'\neq w,\, \textrm{ s.t. }\max_{s< t}\bar\Lambda_s(w,w') < T \}.$$
Then, clearly, $\mathbb{E}(T')\leq \mathbb{E}(T_w)$. Thus $\mathbb{E}(T') \leq \mathbb{E}(\max_{w'\neq w} T^{w'}_{w})\leq \sum_{w'\neq w} \mathbb{E}(T^{w'}_{w}).$
Now, it is easy to show that the process 
$$\bigg(\bar\Lambda_{t}(w,w')-t\mathbb{E}_\cD(\KL(w,w'|\epsilon,c))\bigg)_{t\in\mathbb{N}}$$
is a martingale and $T^{w'}_{w}$ is a bounded stopping time \citep{ross1996stochastic}. Thus, by the optional stopping theorem,
$$\mathbb{E}(\bar\Lambda_{T^{w'}_{w}}(w,w')-T^{w'}_{w}\mathbb{E}_\cD(\KL(w,w'|\epsilon,c)) = 0,$$
that is,
\begin{align*}
\mathbb{E}(T^{w'}_{w}) &= \frac{\mathbb{E}(\bar\Lambda_{T^{w'}_{w}}(w,w'))}{\mathbb{E}_\cD(\KL(w,w'|\epsilon,c))}\\
&\leq  \frac{\log T + \mathbb{E}(\log \frac{p(U_{T^{w'}_{w}}\mid \epsilon,c_{T^{w'}_{w}},w)}{p(U_{T^{w'}_{w}}\mid \epsilon,c_{T^{w'}_{w}},w')})}{L(\epsilon)}~~~\textrm{ (by the definition of }T^{w'}_{w})\\
&\leq \frac{\log T + \mathbb{E}(\log \frac{1}{b\epsilon^H})}{L(\epsilon)}~~~\textrm{(for small enough }\epsilon \textrm{ for some $b>0$, since $D(x)= \Omega(x^H)$)}\\
&\leq \frac{\log T + b' -H\log \epsilon}{L(\epsilon)}~~~\textrm{ (for small enough }\epsilon),
\end{align*}
where $b' = -\log b$. Hence, finally,
\begin{align}
\mathbb{E}(T')\leq (|\cW| -1|)\frac{\log T + b' -H\log \epsilon}{L(\epsilon)}
\end{align}
for small enough $\epsilon$.
%We can show that (Lemma 4.3 in \cite{rajeev1989asymptotically})
%$$\liminf_{t\rightarrow\infty} \frac{\mathbb{E}(T')}{\log T}= \frac{1}{L(\epsilon|w)}\leq  \frac{1}{L(\epsilon)}.$$ 
Next, suppose that $w^*$ is the parameter value that \CFE~learns at the end of Explore. If we denote $P^{err}_w$ to be the probability that $w^*\neq w$ when the true parameter is $w$, then we can show that $P^{err}_w\leq 1/T$. This follows from the fact that under the true $w$, the sequence of likelihood ratios $(\exp(-\Lambda_t(w,w')))_{0\leq t\leq T}$ is a martingale and hence by Doob's martingale inequality \citep{ross1996stochastic},
$\mathbb{P}_{w}(\max_{t\leq T} \Lambda_t(w',w)> \log T)= \mathbb{P}_{w}(\max_{t\leq T} \exp(-\Lambda_t(w,w'))> T)\leq 1/T$
for any $w'\neq w$. 

Finally, if we denote $U^{\epsilon}_K(w)$ to be the optimal value of the optimization problem \eqref{opt2} when $w^*=w$, then we can show that $U^{\epsilon}_K(w)\geq U_K(w) -\epsilon R$. This is because we can take the optimal $K$-Lipschitz decision rule $\phi$ in $\cX=[0,1]$ that attains utility $U_K(w)$, and we can define a new decision rule $\phi'$ such that $\phi'(c) \triangleq \phi(c)$ if $\phi(c)\geq \epsilon$ and $\phi'(c) \triangleq \epsilon$ otherwise. It is easy to verify that this decision rule is $K$-Lipschitz and all the decisions are in $\cX_{\epsilon}$; hence it is feasible in problem \eqref{opt2}. Clearly, the expected utility of this decision rule is at least $U_K(w) - \epsilon R$ because of our assumption that $\bar{u}(x,c,w)$ is Lipschitz continuous in $x$ for each $c$ and $w$, with a Lipschitz constant $R$. 

Thus, we finally have that for a fixed model parameter $w$, the following upper bound holds for the total regret under \CFE:
\begin{align*}
\Reg^T_{K\textup{-FT}}(w, \textup{\CFE})&\leq\underbrace{R\mathbb{E}(T')}_\text{(a)} + \underbrace{P_w(w^*\neq w)(T-\mathbb{E}[T'\mid w^*\neq w])R}_\text{(b)} \\
&~+ \underbrace{P_w(w^*= w)(T-\mathbb{E}[T'\mid w^*= w])R\epsilon}_\text{(c)}\\
&\leq R\mathbb{E}(T') + \frac{1}{T}(TR) + TR\epsilon\\
&= R\mathbb{E}(T') + R+ TR\epsilon\\
&\leq R(|\cW| -1|)\frac{\log T + b' -H\log \epsilon}{L(\epsilon)}+ R+ TR\epsilon
\end{align*}
for some small enough $\epsilon$, where (a) is the upper bound on the regret incurred during Explore, (b) is the upper bound on the regret incurred during Exploit on the event $\{w^*\neq w\}$, and (c) is the upper bound on the regret incurred during Exploit on the event $\{w^*= w\}$.
%At this point, the argument diverges for the two cases. 
%\noindent\textbf{Case 1: $\liminf_{x\rightarrow 0} L(x) = S>0$}
%In this case, we have 
%\begin{align*}
%&\Reg^T_{K\textup{-FT}}(w, \textup{\CFE})\leq R(|\cW| -1|)\frac{\log T + \vj{b'} -H\log \epsilon}{S+h(\epsilon)}+ R+ TR\epsilon
%\end{align*}
%for some $\epsilon$ small enough and $h(x)$ such that $\lim_{\epsilon\rightarrow 0} h(x) = 0$. Plugging in $\epsilon = 1/T$, for a large enough $T$ we have,
%\begin{align*}
%&\Reg^T_{K\textup{-FT}}(w,\textup{\CFE}_T)\leq R(|\cW| -1|)\frac{\log T + \vj{b'} +H\log T}{S+h(1/T)}+ 2R.
%\end{align*}
%Thus,
%$$\limsup_{T\rightarrow\infty}\frac{\Reg^T_{K\textup{-FT}}(w, \textup{\CFE}_T)}{\log T} \leq R(|\cW| -1|)\frac{1+H}{S}.$$
%\noindent\textbf{Case 2: $L(x)=\Omega(x^M)$ for some $M>0$ as $x\rightarrow 0$}
Now since $L(x) =\Omega(x^M)$ as $x \rightarrow 0$, there exists $a>0$ such that $L(x) \geq ax^M$. Hence, we have
\begin{align*}
&\Reg^T_{K\textup{-FT}}(w, \textup{\CFE})\leq R(|\cW| -1|)\frac{\log T + b' -H\log \epsilon}{a\epsilon^M}+ R+ TR\epsilon
\end{align*}
for some small enough $\epsilon$. Plugging in $\epsilon = 1/T^{\frac{1}{M+1}}$, we obtain that for a large enough $T$,
\begin{align*}
&\Reg^T_{K\textup{-FT}}(w, \textup{\CFE}_T)\leq R(|\cW| -1|)T^{\frac{M}{M+1}}\frac{\log T + b' +\frac{H}{M+1}\log T}{a}+ R+ T^{\frac{M}{M+1}}R.
\end{align*}
Thus,
$$\limsup_{T\rightarrow\infty}\frac{\Reg^T_{K\textup{-FT}}(w, \textup{\CFE}_T)}{T^{\frac{M}{M+1}}\log T} \leq  \frac{R(|\cW| -1|)(1 +\frac{H}{M+1})}{a}.$$
This proves the statement of the theorem.
\hfill $\square$
\end{proof}
%\paragraph{Remark.} If there exists a set of optimal K-Lipschitz decision-rules $(\phi^*_w)_{w\in\cW}$ such that 
%$$\min_{c\in\cC, w\in\cW}\phi^*_w = x^*>0,$$
%then, although the performance in Theorem~\ref{thm:main} still holds in this case, one can obtain a better performance in practice
%\paragraph{Accountability of {\sc CaFE}} The general data protection regulation, a regulation in the EU law that came into effect in May 2016, outlines the rights of the individual for which a decision is made by an automated mechanism. These rights include the right of the individual to contest any automated decision-making that was made on a solely algorithmic basis (Article 12 (3)). Fairness-in-hindsight binds these automated decisions to the decisions taken on past contexts, thereby empowering the decision-maker to explain their decisions for any individual using precedence. [VJ: Repeated in the intro. Check again]

%\paragraph{Ensuring partial fairness-in-foresight under {\sc CaFE}} Suppose we additionally want to satisfy fairness-in-foresight for $Y$ number of time periods into the future, at each time. We can slightly modify \CFE~so that after Explore it waits for $Y$ periods before entering the Exploit phase. It would then satisfy both the constraints {\it and} achieve sublinear regret guarantees in the settings that we describe. 

\subsection{Lower bounds on regret under FH}
We now construct examples that demonstrate that the regret under \CFE~has essentially order-optimal dependence on $T$.
%The proofs of these bounds rely on the following high-level argument
%that embodies the spirit of similar arguments for obtaining lower bounds on regret in other types of online learning problems in stochastic settings, e.g., multi-armed bandit problems. 
The proofs of these bounds rely on the following high-level argument. Suppose there are two hypotheses about the state of the world: one in which high decisions are optimal and the other in which low decisions are optimal. If a policy raises decisions too quickly (and importantly, irrevocably, because of the FH constraint) without obtaining sufficiently strong empirical evidence for the hypothesis that it is indeed in a situation where high decisions are optimal, then (a ``change of measure" argument shows that) there is a good chance that the hypothesis is incorrect, and because of the FH constraint, the policy will incur a high regret. The fact that a policy is {\it good}, i.e, it does not incur high regret irrespective of the true hypothesis, implies an upper bound on the probability of this policy raising decisions too quickly without obtaining sufficient empirical evidence to justify it. This means that any good policy must obtain sufficient evidence before raising decisions. But {\it it takes time} to obtain this evidence, which means that it is {\it inevitable} that the policy will incur some regret in the case that high decisions are optimal. In the case where $\lim_{x\rightarrow 0} L(x) = 0$, the situation is worse: while the decisions are low, it takes {\it even longer} to obtain sufficient empirical evidence to justify higher decisions, thus increasing the amount of inevitable regret. 
%[MAKE SURE WE TALK ABOUT HOW THIS IS A NON-TRIVIAL EXTENSION OF TYPICAL LOWER BOUNDS ESP. IN THE HARD SETTING]

\subsubsection{Case: $L(x)=\Theta(x^M)$.}
First, we construct an instance with $L(x) = \Theta(x^M)$ as $x\rightarrow 0$ where a regret of $\omega(T^{\frac{M}{M+1}-\beta})$ is inevitable for any $\beta>0$. Suppose that $\cC=\{0\}$, i.e., there is only one context. Hence the choice of the Lipschitz constant $K$ is immaterial, and the FH constraint simply means that the decisions must be non-decreasing over time. Let $\cW=\{A,B\}$ and let the distribution of utility given the decision $x$ and parameter $w$ be defined as:
$$
U \overset{(w=A)}{=}
\left\{
	\begin{array}{llll}
		&x^{1-M/2}  &\mbox{ w.p. } 0.5(1+x^{M/2}),& \\		
		&-x^{1-M/2} &\mbox{ w.p. }0.5(1-x^{M/2}), 
	\end{array}
\right.
$$ and 
\begin{equation}
U \overset{(w=B)}{=}
\left\{
	\begin{array}{llll}
		&x^{1-M/2} &\mbox{ w.p. } 0.5(1-x^{M/2}),& \\ &-x^{1-M/2} &\mbox{ w.p. } 0.5(1+x^{M/2}).
	\end{array}
\right.\label{inst-dist}
\end{equation}

\vspace{0.1in}
It is convenient to define $p(u\mid x,w)$ as the probability of observing $u\in  \{x^{1-M/2},-x^{1-M/2} \}$ for a fixed $w\in\cW$, and $x\in\cX$.\footnote{Here $p(x^{1-M/2} \mid x, A) = p(-x^{1-M/2} \mid x, B)= 0.5(1+x^{\frac{M}{2}})$, $p(-x^{1-M/2} \mid x, A)  = p(x^{1-M/2} \mid x, B)= 0.5(1-x^{\frac{M}{2}})$ and 0 otherwise.}

%$p(u\mid x,w)$ can be expressed as:
%\begin{align}
%p(u\mid x,w)&\triangleq 0.5(1+x^{\frac{M}{2}})\big(\mathbbm{1}_{\{u=x^{1-\frac{M}{2}} \mbox{ and } w = A\}}+ \mathbbm{1}_{\{u=-x^{1-\frac{M}{2}} \mbox{ and } w=B\}}\big)\nonumber\\
%&~+0.5(1-x^{\frac{M}{2}})\big(\mathbbm{1}_{\{u=x^{1-\frac{M}{2}} \mbox{ and } w = A\}} + \mathbbm{1}_{\{u=-x^{1-\frac{M}{2}} \mbox{ and } w=B\}}\big).\nonumber
%\end{align}}

It is easy to see that the mean utilities are: $\bar{u}(x,A) = x$ and $\bar{u}(x,B)= -x$ (where we have suppressed the dependence on the context).  Clearly, the optimal decision is $x=1$ if $w=A$ and $x=0$ if $w=B$. 
For any $p,q \in (0,1)$, let $D_{\KL}(p\|q)$ denote the the K-L divergence of a Bernoulli($p$) distribution relative to a Bernoulli($q$) distribution, i.e., $D_{\KL}(p\|q) = p\log (p/q) +(1-p)\log ((1-p)/(1-q))$. Recall that $L(x)$ was defined to be the minimum expected KL-divergence over any two parameters $w,w^\prime$ (see \eqref{eq:Lx}). %the definition of $L(x)$ in \eqref{eq:Lx}. 
 For this instance, note that $D_{\KL}(0.5(1+x^{M/2})\|0.5(1-x^{M/2}) = D_{\KL}(0.5(1-x^{M/2})\|0.5(1+x^{M/2}))=  x^{M/2}\log\frac{1+x^{M/2}}{1-x^{M/2}}$, and hence 
\begin{align}
L(x) &= \min(D_{\KL}(0.5(1+x^{M/2})\|0.5(1-x^{M/2})), D_{\KL}(0.5(1-x^{M/2})\|0.5(1+x^{M/2})))\nonumber \\ 
&= x^{M/2}\log\frac{1+x^{M/2}}{1-x^{M/2}}.
\end{align}

Note that $L(x)$ is increasing in $x$ and one can show that $L(x)=\Theta(x^M)$ as $x\rightarrow 0$. This means that for $x$ small enough and for some $0<v<V$, we have 
\begin{align}
vx^M\leq L(x)\leq Vx^M.\label{asym}
\end{align}
\begin{proposition}\label{prop:lb1}
Consider the instance defined above. Suppose there is a sequence of $K$\textup{-FH} policies $(\psi_T)_{T\in\mathbb{N}}$ that satisfy,
 $$\Reg^T_{K\textup{-FH}}(w,\psi_T)= \textup{o}(T^\alpha)$$ 
as $T\rightarrow\infty$ for each $\alpha>\frac{M}{M+1}$ and each $w\in\{A,B\}$. Then 
$$\Reg^T_{K\textup{-FH}}(A,\psi_T)= \omega(T^\beta)$$
as $T\rightarrow\infty$ for each $\beta<\frac{M}{M+1}$.
\end{proposition}
%Consider a sequence of policies $(\Psi_T)_{T\geq1}$, such that $Regret_T(w) = o(T^{\alpha})$ for each $\alpha>T^{2M/2M+1}$ irrespective of $w$. 
\begin{proof}{Proof.}
For any underlying $w$, the policy $\psi_T$ along with the distributions of utilities conditioned on decisions given in \eqref{inst-dist}, induce a probability distribution on the sequence of decision and utility pairs $(X_1,U_1,X_2,U_2,\cdots, X_T,U_T)$. Let $\mathbb{P}_w$ and $\mathbb{E}_w$ denote the probabilities of events and expectations, respectively, under $w=A$ and $w=B$. Here $X_t$ is measurable with respect to the $\sigma$-algebra generated by $(X_1,U_1,X_2,U_2,\cdots, X_{t-1},U_{t-1})$, and $U_t$ is conditionally independent of the past given $X_t$. This, in particular, means that for any function $f:\mathbb{R}\times \cX\rightarrow\mathbb{R}$, and each $w\in\{A,B\}$, 
\begin{equation} \mathbb{E}_w[f(U_t, X_t) \mid X_1,U_1,X_2,U_2,\cdots, X_{t}] = \mathbb{E}_w[f(U_t,X_t) \mid X_{t}].\label{indep}
\end{equation}
Define the (random) sequence of empirical log-likelihood ratios of $w=A$ relative to $w=B$, $(\Lambda_t)_{t\leq T}$, where
\begin{equation}
\Lambda_t = \sum_{s=1}^{t}\log \frac{p(U_s\mid X_s,A)}{p(U_s\mid X_s,B)}.
\end{equation}
Also, we will be using a ``centered" sequence $(\overline{\Lambda}_t)_{0\leq t\leq T}$, where $\overline{\Lambda}_t=\Lambda_t -\mu_t$ and $(\mu_t)_{0\leq t\leq T}$ is the mean process defined as $\mu_0 =0$ and for any $1\leq t\leq T$, 
\begin{equation}
\mu_t = \sum_{s=1}^{t}\mathbb{E}_A[\log \frac{p(U_s\mid X_s,1)}{p(U_s\mid X_s,0)}\mid X_s],\label{eq:mu}
\end{equation}

Using \eqref{indep}, it is easy to see that $(\overline{\Lambda}_t)_{0\leq t\leq T}$ is a martingale. 

The proof is now divided into two parts.

\noindent \textbf{Part A: One cannot raise decisions too quickly without obtaining sufficient empirical evidence to justify it.}

Let us call a policy $\psi_T$ {\it good} if $\Reg^T_{K\textup{-FH}}(B,\psi_T)= \textup{o}(T^\alpha)$ as $T\rightarrow\infty$ for each $\alpha>\frac{M}{M+1}$. In this part, we will define a threshold $\tau(T) \in \cX = [0,1]$. When $w=A$, we will show that under any good policy $\psi_T$, the probability of raising the decision beyond $\tau(T)$ despite the fact that there is insufficient evidence for $w=A$ (i.e., log-likelihood ratio of $A$ relative to $B$ is low) is $\textup{o}(1)$.

Let the threshold be $\tau(T) = T^{\frac{M}{M+1} +\gamma -1}$ for any $\gamma \in (0,1/(M+1))$. Note that $\tau(T)=\textup{o}(1)$ as $T \rightarrow \infty$. Define $k_T$ to be the first time that the policy raises decisions higher than $\tau(T)$ ($k_T = T$ if it doesn't), i.e., $k_T = \max\{1\leq k\leq T \textrm{ s.t. } X_k < \tau(T)\}$. Define the event:
\begin{align*}
&C_T \triangleq\{k_T \leq \frac{\gamma\log T}{2L(\tau(T))} \textrm{ and }\Lambda_{k_T} \leq \frac{3\gamma}{4}\log T\}.
 \end{align*}
Informally, this event says that the decision crosses $\tau(T)$ before time $\frac{\gamma\log T}{2L(\tau(T))}$ and at the time of crossing, the empirical log-likelihood ratio of $A$ relative to $B$ is low, i.e., it is below $\frac{3\gamma}{4}\log T$.  
Then we have
\begin{equation}
\mathbb{P}_B(C_T) =\mathbb{E}_B(\mathbbm{1}_{C_T}) \overset{(a)}{=} \mathbb{E}_A(\mathbbm{1}_{C_T}\exp(-\Lambda_{k_T})) \geq \mathbb{P}_A(C_T)T^{-3\gamma/4}.\label{com1}
\end{equation}
%By a standard change of measure argument, we can then show that 
%$$\mathbb{P}_B(C_T)
Here, $(a)$ is the standard change of measure identity. Hence,
\begin{align*}
\Reg^T_{K\textup{-FH}}(B,\psi_T)&\overset{(a)}{\geq} \mathbb{P}_B(C_T)\tau(T)\bigg(T - k_T\bigg)\\
&\overset{(b)}{\geq} \mathbb{P}_B(C_T)\tau(T)\bigg(T - \frac{\gamma\log T}{2L(\tau(T))}\bigg)\\
&\overset{(c)}{\geq} \mathbb{P}_A(C_T)T^{-3\gamma/4}\tau(T)\bigg(T - \frac{\gamma\log T}{2L(\tau(T))}\bigg)\\
&\overset{(d)}{\geq} \mathbb{P}_A(C_T)T^{-3\gamma/4}\tau(T)\bigg(T - \frac{\gamma\log T}{2v\,\,\tau(T)^{M}}\bigg) 
\end{align*}
for a large enough $T$. Here $(a)$ follows from the fact that when $w=B$, then on the event $C_T$, the policy will incur a regret of at least $\tau(T)$ in each time period after time $k_T$. (b) is simply using the definition of $C_T$, $(c)$ follows from \eqref{com1}, and $(d)$ follows from \eqref{asym}.  Now, since, 
\begin{align*}
\frac{1}{\tau(T)^{M}} &= T^{(-M^2/M+1) -M\gamma +M}\leq T^{(-M^2/M+1) +M} = T^{\frac{M}{M+1}},
\end{align*}
we have that $\frac{\gamma\log T}{2v\tau(T)^{M}}= \textup{o}(T)$.
Thus, we finally have,
\begin{align*}
\Reg^T_{K\textup{-FH}}(B,\psi_T)&\geq \mathbb{P}_A(C_T)T^{-3\gamma/4+(\frac{M}{M+1}) +\gamma -1}(T-\textup{o}(T))\\
&= \mathbb{P}_A(C_T)T^{\gamma/4+(\frac{M}{M+1})}(1-\textup{o}(1)). 
\end{align*}
Since $\Reg^T_{K\textup{-FH}}(B,\psi_T)= \textup{o}(T^\alpha)$ for each $\alpha>\frac{M}{M+1}$, we have $\mathbb{P}_A(C_T) = \textup{o}(1)$.
Thus, for $\overline{C}_T$, the complement of the event $C_T$, we get:
\begin{align}
\mathbb{P}_A(\overline{C}_T)&= \mathbb{P}_A(\underbrace{k_T> \frac{\gamma\log T}{2L(\tau(T))}}_{(\star)}\textrm{ or } \underbrace{\Lambda_{k_T} > \frac{3\gamma}{4}\log T}_{(\dagger)}) = 1-\textup{o}(1).\label{eq:part-A}
\end{align}
This shows that when $w=A$, any good policy must wait for a sufficiently long time before raising the decisions beyond $\tau(T)$ ($\star$), {\it or} there must be sufficient empirical evidence for $w=A$ at the time when the decision is raised beyond $\tau(T)$ $(\dagger)$. We next show that it is inevitable that one has to wait sufficiently long before raising the decision beyond $\tau(T)$. This will follow from the fact that it takes time for the log-likelihood ratio of $w=A$ relative to $w=B$ to grow sufficiently. \\% time required to gain sufficient evidence for $w=A$. %This implies that it is inevitable that one has to wait sufficiently long before raising the decision beyond $\tau(T)$.

\noindent \textbf{Part B: It takes time (and hence regret) to gather sufficient empirical evidence.}\\

We now show that $\displaystyle \mathbb{P}_A(k_T\leq \frac{\gamma\log T}{2L(\tau(T))} \textrm{ and }\Lambda_{k_T} > \frac{3\gamma}{4}\log T) = \textup{o}(1).$ Denote $z(T) = \lfloor{\frac{\gamma\log T}{2L(\tau(T))}}\rfloor$. Then we have,
\begin{align*}
\mathbb{P}_A(k_T\leq \frac{\gamma\log T}{2L(\tau(T))} \textrm{ and }\Lambda_{k_T} > \frac{3\gamma}{4}\log T)&= \mathbb{P}_A(k_T\leq z(T) \textrm{ and }\Lambda_{\min(k_T,z(T))} > \frac{3\gamma}{4}\log T)\\
&\leq \mathbb{P}_A(\Lambda_{\min(k_T,z(T))} > \frac{3\gamma}{4}\log T)\\
&= \mathbb{P}_A(\overline{\Lambda}_{\min(k_T,z(T))} > \frac{3\gamma}{4}\log T- \mu_{\min(k_T,z(T))}  )\\
&\overset{(a)}{\leq} \mathbb{P}_A(\overline{\Lambda}_{\min(k_T,z(T))} > \frac{3\gamma}{4}\log T- z(T)L(\tau(T)) )\\
&\leq \mathbb{P}_A(\overline{\Lambda}_{\min(k_T,z(T))} > \frac{\gamma}{4}\log T ).
%&\leq \frac{\mathbb{E}(|\Lambda_{\min(k_T,z(T))} - mean|} )
\end{align*}
Here, (a) follows from the definition of $\mu_t$ in equation $\eqref{eq:mu}$ and from the fact that 
$$\mathbb{E}_A[\log \frac{p(U_s\mid X_s,1)}{p(U_s\mid X_s,0)}\mid X_s]= L(X_s)\leq L(\tau(T))$$
almost surely for all $s\leq k_T$, since $X_s< \tau(T)$ for all $s\leq k_T$, and $L(\cdot)$ is increasing.

We now define a new policy $\psi'_T$ with associated random variables that will be differentiated from the corresponding random variables under $\psi_T$ by adding a ``prime'' superscript, e.g., $X\rightarrow X'$. This new policy follows the prescriptions of $\psi_T$ until one of the two events happen:
\begin{enumerate}
\item  $\overline{\Lambda}'_{t} > \frac{\gamma}{4}\log T$, in which case it increases decision to $\tau(T)$ and chooses $\tau(T)$ until the end of the horizon.
\item  $\psi_T$ prescribes raising the decision from some $x' < \tau(T)$ to some $x''\geq \tau(T)$, in which case it continues to play $x'$ until either condition (1) is satisfied, or until the end of the horizon. 
\end{enumerate}
Effectively, this policy raises decision to $\tau(T)$ exactly when $\Lambda'_{t} > \frac{\gamma}{4}\log T$ and then fixes the decisions at $\tau(T)$. Define the random variable \begin{align}
k'_T &\triangleq \max\{k\leq T \textrm{ s.t. } X'_k < \tau(T)\}=  \max\{k\leq T \textrm{ s.t. } \sup_{s< k}\overline{\Lambda}'_{s} \leq \frac{\gamma}{4}\log T\}.
\end{align}
Suppose that the sample paths under the two policies are coupled until the point that these two policies have identical prescriptions. 
Then, by the construction of $\psi'_T$, it is clear that on each sample path that $\overline{\Lambda}_{\min(k_T,z(T))} > \frac{\gamma}{4}\log T$ occurs, $\overline{\Lambda}'_{\min(k'_T,z(T))} > \frac{\gamma}{4}\log T$ occurs as well. 
%Let $\mathbb{P}'_1$ and $\mathbb{E}'_1$ denote probabilities of events and expectations respectively under policy $\psi'_T$ (and $w=1$). 
Thus, 
\begin{align*}
\mathbb{P}_A(\overline{\Lambda}_{\min(k_T,z(T))} > \frac{\gamma}{4}\log T ) &\leq \mathbb{P}_A(\overline{\Lambda}'_{\min(k'_T,z(T))}  > \frac{\gamma}{4}\log T )\\
&= \mathbb{P}_A(k'_T \leq z(T)) \\
&\leq \mathbb{P}_A(\sup_{s\leq z(T)}\overline{\Lambda}'_{s}>  \frac{\gamma}{4}\log T)\\
&\leq \frac{\textup{var}(\overline{\Lambda}'_{z(T)})}{(\log T)^2} \textrm{ (by Kolmogorov's maximal inequality \citep{ross1996stochastic})} \\
%&= \frac{\mathbb{E}[\textup{var}(\overline{\Lambda}'_{z(T)}\mid X_1,\cdots,X_{z(T)})]  }{(\log T)^2} \\
%&~~+\frac{\textup{var}(\mathbb{E}(\overline{\Lambda}'_{z(T)}\mid X_1,\cdots,X_{z(T)})}{(\log T)^2}\\
&= \frac{\sum_{t=1}^{z(T)}\textup{var}\big(\log\frac{p(U'_t|X'_t,w)}{p(U'_t\mid X'_t,w')}-\mathbb{E}_A[\log\frac{p(U'_t|X'_t,w)}{p(U'_t\mid X'_t,w')}\mid X'_t]\big)}{(\log T)^2}\\
%&= \frac{\sum_{t=1}^{z(T)}E\big[\textup{var}\big(\log\frac{p(U'_t|X'_t,w)}{p(U'_t\mid X'_t,w')}-\mathbb{E}[\log\frac{p(U'_t|X'_t,w)}{p(U'_t\mid X'_t,w')}\mid X'_t]\mid X'_t\big)\big]}{(\log T)^2}\\
%&= \frac{\sum_{t=1}^{z(T)}E\big[E\big[\big(\log\frac{p(U'_t|X'_t,w)}{p(U'_t\mid X'_t,w')}-\mathbb{E}[\log\frac{p(U'_t|X'_t,w)}{p(U'_t\mid X'_t,w')}\mid X'_t]\big)^2\mid X'_t\big]\big]}{(\log T)^2}\\
&= \frac{\sum_{t=1}^{z(T)}\mathbb{E}_A\big(\textup{var}\big(\log\frac{p(U'_t|X'_t,w)}{p(U'_t\mid X'_t,w')}\mid X'_t\big)\big)}{(\log T)^2}
\end{align*}
The random variable $\log\frac{p(U'_t\mid X'_t,w)}{p(U'_t\mid X'_t,w')}$ lies in $[-\log\frac{1+(X'_t)^{M/2}}{1-(X'_t)^{M/2}}, \log\frac{1+(X'_t)^{M/2}}{1-(X'_t)^{M/2}}]$. Under policy $\psi'_T$, $X'_t\leq \tau(T)$ almost surely for all $t\leq T$. Thus the range of $\log\frac{p(U'_t\mid X'_t,w)}{p(U'_t\mid X'_t,w')}$ is at most $[-\log\frac{1+\tau(T)^{M/2}}{1-\tau(T)^{M/2}}, \log\frac{1+\tau(T)^{M/2}}{1-\tau(T)^{M/2}}]$. Hence, by Popoviciou's inequality for the variances,\footnote{If a random variable takes values in $[a,b]$, then its variance is at most $(b-a)^2/4$.}
\begin{align*}
\textup{var}\big(\log\frac{p(U'_t\mid X'_t,w)}{p(U'_t\mid X'_t,w')}\mid X'_t\big)&\leq (\log\frac{1+\tau(T)^{M/2}}{1-\tau(T)^{M/2}})^2 = \textup{O}(\tau(T)^M)
\end{align*}
as $\tau(T)\rightarrow 0$. Hence, we finally have 
\begin{align*}
\mathbb{P}_A(\overline{\Lambda}'_{\min(k_T,z(T))}> \frac{\gamma}{4}\log T )\leq \frac{z(T)\textup{O}(\tau(T)^M)}{(\log T)^2}\leq \frac{\gamma\textup{O}(\tau(T)^M)}{L(\tau(T))\log T} \overset{(a)}{\leq} \frac{\gamma\textup{O}(\tau(T)^M)}{v\tau(T)^M \log T} = \textup{o}(1), 
\end{align*}
where (a) follows from \eqref{asym}. 
Thus, to reiterate, we have shown that
$$\mathbb{P}_A(k_T\leq \frac{\gamma\log T}{2L(\tau(T))} \textrm{ and }\Lambda_{k_T} > \frac{3\gamma}{4}\log T) = \textup{o}(1).$$
Coupled with \eqref{eq:part-A}, this implies that $\mathbb{P}_A(k_T>\frac{\gamma\log T}{2L(\tau(T))})= 1-\textup{o}(1).$
Hence,
\begin{align*}
\Reg^T_{K\textup{-FH}}(1,\psi_T)&\geq (1-\textup{o}(1))(1-\tau(T))\frac{\gamma\log T}{2L(\tau(T))}\\
&\geq (1-\textup{o}(1))(1-\textup{o}(1))\frac{\gamma\log T}{2V\tau(T)^{M}}\textrm{ (for a large enough $T$)}\\
&\geq  (1-\textup{o}(1)) \frac{\gamma}{2V} T^{\frac{M}{M+1} -M\gamma}\log T \textrm{ (for a large enough $T$)}.
\end{align*}
Thus $\Reg^T_{K\textup{-FH}}(1,\psi_T) = \omega(  T^{\frac{M}{M+1} -M\gamma})$ for every $\gamma>0$, hence proving the claim.\hfill $\square$
\end{proof}

\paragraph{Remark.} If there was no FH constraint, a cumulative regret of at most 1 can be achieved in this instance: one can simply choose $x_1=1$, which immediately reveals whether $w=1$ (if $U_1 = 1$) or $w=0$ if ($U_1 = -1$). This demonstrates the stark impact of the FH constraint on regret.

This lower bound, however, does not resolve whether in the setting where $L(x) = \Theta(1)$, a regret of $\textup{O}(\log T)$ is necessary, as our upper bound suggests. We show next that this is indeed the case. 

%The idea of the proof is as follows. At time $k$, if $u_k \geq f(k)$, 
%Define the event 
%$$C_n \triangleq \{k\leq (1-\epsilon)\log n/KL(0,1) \textrm{ and } u_k > 1/\sqrt{n} \textrm{ and }\hat{KL}_k \leq (1-\epsilon/2)\log n\}.$$
%$$\mathbb{P}'(C_n)\geq \mathbb{P}(C_n)n^{1-\epsilon/2}.$$
%$$Regret'\geq \mathbb{P}(C_n)\frac{n^{1-\epsilon/2}}{\sqrt{n}}(n - \frac{(1-\epsilon)\log n}{KL(0,1)})$$
%Thus $\mathbb{P}(C_n) = o(1)$.
%Thus 
%$$\mathbb{P}(C'_n)= \mathbb{P}(\{k\geq (1-\epsilon)\log n/KL(0,1) \textrm{ or } u_k \leq f(n) \textrm{ or }\hat{KL}_k \geq (1-\epsilon/2)\log n \}) = 1-o(1).$$
%Note that 
%$$\mathbb{P}(\{k\geq (1-\epsilon)\log n/KL(0,1) \textrm{ or } u_k \leq f(n) \textrm{ or }\hat{KL}_k \geq (1-\epsilon/2)\log n \})$$
%$$= \mathbb{P}(\{k< (1-\epsilon)\log n/KL(0,1) \textrm{ and } u_k > f(n) \textrm{ and }\hat{KL}_k \geq (1-\epsilon/2)\log n \})$$
%$$+ \mathbb{P}(\{k\geq (1-\epsilon)\log n/KL(0,1) \textrm{ or } u_k \leq f(n) \}).$$
%We will next show that 
%$\mathbb{P}(\{k< (1-\epsilon)\log n/KL(0,1) \textrm{ and }\hat{KL}_k \geq (1-\epsilon/2)\log n \}) = o(1)$. This would imply that 
%$$\mathbb{P}(\{k\geq (1-\epsilon)\log n/KL(0,1) \textrm{ or } u_k \leq f(n) \})= 1-\textup{o}(1).$$
%Hence
%$$Regret= (1-o(1))(1-f(n)) (1-\epsilon)\log n/KL(0,1).$$
%ATTEMPT 2. 
%To see this, suppose that $X_1=1$, then this forces $X_t =1$ whenever $c_t = 0$, which incurs a strict long-run loss relative to $U_{s}(1/2,1)$ when $w=1$. If on the other hand $X_1=0$, then $X_t =0$ whenever $c_t = 0$, which incurs a strict long-run loss relative to $U_{s}(1/2,0)$ when $w=0$. 

\subsubsection{Case: $L(x) =\Theta(1)$.}
We now construct an instance with $L(x) =\Theta(1)$ as $x\rightarrow 0$ where an expected regret of $\Omega(\log T)$ is inevitable. Suppose that $\cC=\{0\}$ and let $\cW=\{A,B\}$. Again, since there is only one context, the choice of the Lipschitz constant $K$ is immaterial, and the FH constraint simply means that the decisions must be non-decreasing over time. The utility at time $t$ given the decision $x_t$ and parameter $w$ is given by $U_t =x_tF_t$ where $F_t$ is i.i.d. across time, distributed as:\\

$\textcolor{white}{ccccccccc} F_t \overset{(w=A)}{=}
\left\{
	\begin{array}{llll}
		&1  &\mbox{ w.p. } 0.75,& \\		&-1 &\mbox{ w.p. }0.25&
	\end{array}
\right. 
\mbox{ and     }F_t \overset{(w=B)}{=}
\left\{
	\begin{array}{llll}
		&1&\mbox{ w.p. } 0.25,& \\ 
		&-1&\mbox{ w.p. } 0.75&
	\end{array}
\right..$\\

Clearly, the optimal decision is $x=1$ if $w=A$ and $x=0$ if $w=B$.

\begin{proposition}
Consider the instance defined above. Suppose there is a sequence of $K$\textup{-FH} policies $(\psi_T)_{T\in\mathbb{N}}$ that satisfy,
 $$\Reg^T_{K\textup{-FH}}(w,\psi_T)= \textup{o}(T^\alpha)$$ 
as $T\rightarrow \infty$ for each $\alpha>0$ and each $w\in\{A,B\}$. Then as $T\rightarrow \infty$, 
$$\Reg^T_{K\textup{-FH}}(A,\psi_T)\geq \Omega(\log T).$$

\end{proposition}

The structure of the proof for this case is similar to that of Proposition 3; in fact, the arguments are simpler since the KL-divergence of $w=A$ relative to $w=B$ (or vice-versa) is independent of the decisions taken by the policy for non-zero decisions. We detail the proof below for completeness. 

\begin{proof}{Proof.}
The policy $\psi_T$ along with the distribution of $F_t$ induces a probability distribution on the sequence of decision and utility pairs $(X_1,U_1,X_2,U_2,\cdots, X_T,U_T)$. Here $X_t$ is measurable with respect to the $\sigma$-algebra generated by $(X_1,U_1,X_2,U_2,\cdots, X_{t-1},U_{t-1})$, and $U_t$ is conditionally independent of the past given $X_t$. Let $\mathbb{P}_w$ and $\mathbb{E}_w$ denote the probabilities of events and expectations, respectively, under $w=A$ and $w=B$.
%This in particular means that for any function $f:\cU\rightarrow\mathbb{R}$, 
%$$\mathbb{E}[f(U_t) \mid X_1,U_1,X_2,U_2,\cdots, X_{t}] = \mathbb{E}[f(U_t) \mid X_{t}]$$ 

Define the (random) sequence of empirical log-likelihood ratios of $w=A$ relative to $w=B$, $(\Lambda_t)_{0\leq t\leq T}$, where $\Lambda_0=0$ and for $1\leq t\leq T$,
\begin{equation}
\Lambda_t = \sum_{s=1}^{t}\mathbbm{1}_{\{X_t>0\}}\log \frac{0.75\mathbbm{1}_{\{F_t = 1\}}+0.25\mathbbm{1}_{\{F_t = -1\}}}{0.75\mathbbm{1}_{\{F_t = -1\}}+0.25\mathbbm{1}_{\{F_t = 1\}}}.
\end{equation}
This is the sequence seen by the policy whenever a non-zero decision is taken (otherwise $F_t$ is not observed). We denote the {\it complete} sequence of empirical log-likelihood ratios by $(\Lambda^c_t)_{0\leq t\leq T}$, where $\Lambda^c_0=0$ and for $1\leq t\leq T$,
\begin{equation}
\Lambda^c_t = \sum_{s=1}^{t}\log \frac{0.75\mathbbm{1}_{\{F_t = 1\}}+0.25\mathbbm{1}_{\{F_t = -1\}}}{0.75\mathbbm{1}_{\{F_t = -1\}}+0.25\mathbbm{1}_{\{F_t = 1\}}}.
\end{equation}
Note that this sequence only depends on $F_t$ and is independent of the decisions $X_t$ taken by the policy. Similar to Proposition~\ref{prop:lb1}, the proof is now divided into two parts.\\

\noindent \textbf{Part A: One cannot raise decisions too quickly without obtaining sufficient empirical evidence to justify it.}\\

%We call a policy $\psi_T$ {\it good} if $\Reg^T_{K\textup{-FH}}(B,\psi_T)= \textup{o}(T^\alpha)$ as $T\rightarrow\infty$ for each $\alpha>0$. When $w=A$, we will show that under any good policy $\psi_T$, the probability of raising the decision beyond a carefully defined threshold, despite the fact that there is insufficient evidence for $w=A$ (i.e., log-likelihood ratio of $A$ relative to $B$ is low), is $\textup{o}(1)$. 

Define a threshold $\tau(T) = 1/T^{\gamma/4}$ for $\gamma>0$ and let $k_T = \max\{1\leq k\leq T \textrm{ s.t. } x_k < 1/T^{\gamma/4}\}$. Consider the following event: 
\begin{align}
&C_T \triangleq\bigg\{k_T \leq \frac{(1-\gamma)\log T}{D_{\KL}(0.75\|0.25)} \textrm{ and }\Lambda_{k_T} \leq (1-\frac{\gamma}{2})\log T\bigg\}.
 \end{align}
%Let $P_w$ and $\mathbb{E}_A$ denote the probabilities of events and expectations, respectively, under $w=0$ and $w=1$. 
By the change of measure identity, we have that 
\begin{align}
\mathbb{P}_B(C_T)=\mathbb{E}_B(\mathbbm{1}_{C_T})= \mathbb{E}_A(\mathbbm{1}_{C_T}\exp(-\Lambda_{k_T})) \geq \mathbb{P}_A(C_T)T^{-1+\gamma/2}.\label{com2}
\end{align}
Hence, \begin{align*}
\Reg^T_{K\textup{-FH}}(B,\psi_T)&\overset{(a)}{\geq}\mathbb{P}_B(C_T)\frac{1}{T^{\gamma/4}}(T - k_T)\\
&\overset{(b)}{\geq} \mathbb{P}_A(C_T)\frac{T^{-1+\gamma/2}}{T^{\gamma/4}}(T - k_T)\\
&\overset{(c)}{\geq} \mathbb{P}_A(C_T)\frac{T^{-1+\gamma/2}}{T^{\gamma/4}}(T - \frac{(1-\gamma)\log T}{D_{\KL}(0.75\|0.25)}).
\end{align*}
Here, (a) follows from the fact that for $w=B$, on event $C_T$, one incurs a regret of $T^{-\gamma/4}$ per time step after $k_T$. (b) follows from \eqref{com2}, and (c) follows from the fact that on event $C_T$, $k_T\leq \frac{(1-\gamma)\log T}{D_{\KL}(0.75\|0.25)}  $. Since $\Reg^T_{K\textup{-FH}}(w,\psi_T)= \textup{o}(T^\alpha)$ for each $\alpha>0$ and for each $w$, we have $\mathbb{P}_A(C_T) = \textup{o}(1)$. Thus, if we denote $\overline{C}_T$ to be the complement of the event $C_T$, then we have,
\begin{align}\label{interim0}
\mathbb{P}_A(\overline{C}_T)&= \mathbb{P}_A(k_T> \frac{(1-\gamma)\log T}{D_{\KL}(0.75\|0.25)}\textrm{ or }\Lambda_{k_T} > (1-\frac{\gamma}{2})\log T ) = 1-o(1).
\end{align}

%Note that 
%$$\mathbb{P}(k(n)\geq (1-\epsilon)\log n/KL(0,1) \textrm{ or }\hat{KL}_k \geq (1-\epsilon/2)\log n \})$$
%$$= \mathbb{P}(\{k(n)< (1-\epsilon)\log n/KL(0,1) \textrm{ and }\hat{KL}_k \geq (1-\epsilon/2)\log n \})$$
%$$+ \mathbb{P}(\{k(n)\geq (1-\epsilon)\log n/KL(0,1) \}).$$

\noindent \textbf{Part B: It takes time (and hence regret) to gather sufficient empirical evidence.}\\

Next, denote $b(T) = \lfloor\frac{(1-\gamma)\log T}{D_{\KL}(0.75\|0.25)}\rfloor$. Then 
\begin{align}
\mathbb{P}_A(k_T\leq \frac{(1-\gamma)\log T}{D_{\KL}(0.75\|0.25)}\textrm{ and }\Lambda_{k_T} > (1-\frac{\gamma}{2})\log T )&= \mathbb{P}_A(k_T\leq b(T)\textrm{ and }\Lambda_{k_T} > (1-\frac{\gamma}{2})\log T) \nonumber\\
&\leq \mathbb{P}_A(k_T\leq b(T) \textrm{ and }\sup_{t\leq b(T)}\Lambda_{t} > (1-\frac{\gamma}{2})\log T) \nonumber\\
&\leq \mathbb{P}_A(\sup_{t\leq b(T)}\Lambda_{t} >(1-\frac{\gamma}{2})\log T) \nonumber\\
&\leq \mathbb{P}_A(\sup_{t\leq b(T)}\Lambda^c_{t} >(1-\frac{\gamma}{2})\log T)  \nonumber\\
&\leq \mathbb{P}_A(\frac{1}{b(T)}\sup_{t\leq b(T)}\Lambda^c_{t} > \frac{(1-\frac{\gamma}{2})}{1-\gamma}D_{\KL}(0.75\|0.25))\nonumber\\
& = \textup{o}(1).\label{interim1}
\end{align}
The last inequality results from the fact that, by the maximal version of the strong law of large numbers (see Theorem 2.2 \cite{bubeck2012regret}), 
$$\frac{1}{b(T)}\sup_{t\leq b(T)}\Lambda^c_{t} \overset{a.s.}{\longrightarrow}D_{\KL}(0.75\|0.25) \textrm{ as } T\rightarrow\infty.$$
%Hence, 
%$$\mathbb{P}_A(\frac{1}{b(T)}\sup_{t\leq b(T)}\Lambda^c_{t} > \frac{(1-\frac{\gamma}{2})}{1-\gamma}D_{\KL}(0.75\|0.25))=\textup{o}(1)$$
%The last claim follows from the maximal version of the strong law of large numbers; see Theorem 2.2 \cite{bubeck2012regret}).
Combining \eqref{interim0} and \eqref{interim1}, we finally have,
$$\mathbb{P}_A(k_T\geq \frac{(1-\gamma)\log T}{D_{\KL}(0.75\|0.25)})= 1-\textup{o}(1).$$
Hence, 
$$\Reg^T_{K\textup{-FH}}(A,\psi_T)\geq(1-o(1))(1-\frac{1}{T^{\gamma/4}})  \frac{(1-\gamma)\log T}{D_{\KL}(0.75\|0.25)}.$$
%The per-period expected loss in utility relative to $U_{K}(w^*)$ from that point on is at most  $\epsilon|\cC|\max_{x\in\cX;\,c\in\cC;\,w'\in\cW}|g(x,c,w')|$. These facts together imply the result.
This implies the result.\hfill $\square$
\end{proof}

\paragraph{Remark.} In the absence of FH constraint, we can show that a $\textup{O}(1)$ cumulative regret can be guaranteed in the instance above. This can, for instance, be achieved by the following policy that always chooses $x_t>0$. Let $\bar{F}_t = (1/t)\sum_{s=1}^t U_t/x_t=(1/t) \sum_{s=1}^t F_t$. Then choose $x_1 = 1$ and for $t\geq 2$, choose $x_t = 1$ if $\bar{F}_{t-1} \geq 0$ and $x_t = e^{-t}$ if $\bar{F}_{t-1} < 0$. Thus the total expected regret on the event $w=1$ is upper bounded by 
$\sum_{t=1}^{T}\mathbb{P}_A(\bar{F}_{t-1} < 0),$
and the total expected regret on the event $w=1$ is upper bounded by 
$$\sum_{t=1}^{T}e^{-t}\mathbb{P}_B(\bar{F}_{t-1} < 0)+\mathbb{P}_B(\bar{F}_{t-1} \geq 0)\leq \sum_{t=1}^{T}e^{-t}+\mathbb{P}_B(\bar{F}_{t-1} \geq 0).$$
Both these quantities are $\textup{O}(1)$ as $T\rightarrow\infty$ since by standard Hoeffding bounds, $\mathbb{P}_B(\bar{F}_{t} \geq 0)\leq e^{-\nu_1 t}$ and $\mathbb{P}_A(\bar{F}_{t} < 0)\leq e^{-\nu_2 t}$ for some instance dependent constants $\nu_1,\, \nu_2>0$.

\section{Conclusion}\label{sec:conclusion}
In this paper, we proposed a new notion of fairness, fairness-in-hindsight, that extends the concept of individual fairness under temporal considerations. Our proposal is simple, intuitive, and importantly, we show that it assimilates well with sequential decision-making algorithms that involve learning, unlike the more straightforward notion of fairness-across-time. This latter aspect inspires optimism by suggesting that similar temporal fairness notions already embedded in our critical societal systems like law need not necessarily hinder learning good policies over time -- as we pointed out earlier, conservative exploration then exploitation structure of our fair-in-hindsight learning algorithm \CFE~is already observed in these contexts. Finally, fairness-in-hindsight can be a practical, first-order safeguard against claims of discrimination in modern algorithmic deployments.

In our proposed model, we assume that the model parameters come from a finite discrete set $\cW$. Also, in many real-world scenarios, the distribution of utilities might change over time. Extending these results to bounded model parameters (not necessarily finite) and changing utilities would be interesting extensions of this work.  

\bibliographystyle{apalike}
\bibliography{ref.bib}

\begin{thebibliography}{}

\bibitem[Angwin et~al., 2016]{Angwin2016}
Angwin, J., Larson, J., Mattu, S., and Kirchner, L. (2016).
\newblock Machine bias: There's software used across the country to predict
  future criminals. and it's biased against blacks.
\newblock {\em ProPublica}.

\bibitem[Bolton et~al., 2003]{bolton2003consumer}
Bolton, L.~E., Warlop, L., and Alba, J.~W. (2003).
\newblock Consumer perceptions of price (un) fairness.
\newblock {\em Journal of consumer research}, 29(4):474--491.

\bibitem[Bubeck et~al., 2012]{bubeck2012regret}
Bubeck, S., Cesa-Bianchi, N., et~al. (2012).
\newblock Regret analysis of stochastic and nonstochastic multi-armed bandit
  problems.
\newblock {\em Foundations and Trends{\textregistered} in Machine Learning},
  5(1):1--122.

\bibitem[Calders and Verwer, 2010]{calders2010three}
Calders, T. and Verwer, S. (2010).
\newblock Three naive bayes approaches for discrimination-free classification.
\newblock {\em Data Mining and Knowledge Discovery}, 21(2):277--292.

\bibitem[Celis et~al., 2018]{celis2018algorithmic}
Celis, L.~E., Kapoor, S., Salehi, F., and Vishnoi, N.~K. (2018).
\newblock An algorithmic framework to control bias in bandit-based
  personalization.
\newblock {\em arXiv preprint arXiv:1802.08674}.

\bibitem[Chouldechova, 2017]{chouldechova2017fair}
Chouldechova, A. (2017).
\newblock Fair prediction with disparate impact: A study of bias in recidivism
  prediction instruments.
\newblock {\em Big data}, 5(2):153--163.

\bibitem[Chouldechova and G'Sell, 2017]{chouldechova2017fairer}
Chouldechova, A. and G'Sell, M. (2017).
\newblock Fairer and more accurate, but for whom?
\newblock {\em arXiv preprint arXiv:1707.00046}.

\bibitem[Corbett-Davies and Goel, 2018]{Corbett2018}
Corbett-Davies, S. and Goel, S. (2018).
\newblock The measure and mismeasure of fairness: A critical review of fair
  machine learning.
\newblock {\em arXiv preprint arXiv:1808.00023}.

\bibitem[Dwork et~al., 2012]{Dwork2012}
Dwork, C., Hardt, M., Pitassi, T., Reingold, O., and Zemel, R. (2012).
\newblock Fairness through awareness.
\newblock In {\em {Proceedings of the 3rd Innovations in Theoretical Computer
  Science conference}}, pages 214--226. ACM.

\bibitem[Dwork and Ilvento, 2018]{Dwork2018individual}
Dwork, C. and Ilvento, C. (2018).
\newblock Individual fairness under composition.

\bibitem[Elzayn et~al., 2018]{elzayn2018fair}
Elzayn, H., Jabbari, S., Jung, C., Kearns, M., Neel, S., Roth, A., and
  Schutzman, Z. (2018).
\newblock Fair algorithms for learning in allocation problems.
\newblock {\em arXiv preprint arXiv:1808.10549}.

\bibitem[Friedland, 1974]{friedland1974prospective}
Friedland, M.~L. (1974).
\newblock Prospective and retrospective judicial lawmaking.
\newblock {\em The University of Toronto Law Journal}, 24(2):170--190.

\bibitem[Gillen et~al., 2018]{gillen2018online}
Gillen, S., Jung, C., Kearns, M., and Roth, A. (2018).
\newblock Online learning with an unknown fairness metric.
\newblock {\em arXiv preprint arXiv:1802.06936}.

\bibitem[Hardt et~al., 2016]{Hardt2016}
Hardt, M., Price, E., Srebro, N., et~al. (2016).
\newblock Equality of opportunity in supervised learning.
\newblock In {\em {Advances in Neural Information Processing Systems (NIPS)}},
  pages 3315--3323.

\bibitem[Heidari and Krause, 2018]{heidari-2018}
Heidari, H. and Krause, A. (2018).
\newblock Preventing disparate treatment in sequential decision making.
\newblock In {\em Proceedings of the Twenty-Seventh International Joint
  Conference on Artificial Intelligence, {IJCAI-18}}, pages 2248--2254.
  International Joint Conferences on Artificial Intelligence Organization.

\bibitem[Jabbari et~al., 2017]{jabbari2017fairness}
Jabbari, S., Joseph, M., Kearns, M., Morgenstern, J., and Roth, A. (2017).
\newblock Fairness in reinforcement learning.
\newblock In {\em Proceedings of the 34th International Conference on Machine
  Learning-Volume 70}, pages 1617--1626.

\bibitem[Joseph et~al., 2016a]{joseph2016fair}
Joseph, M., Kearns, M., Morgenstern, J., Neel, S., and Roth, A. (2016a).
\newblock Fair algorithms for infinite and contextual bandits.
\newblock {\em arXiv preprint arXiv:1610.09559}.

\bibitem[Joseph et~al., 2016b]{joseph2016fairness}
Joseph, M., Kearns, M., Morgenstern, J.~H., and Roth, A. (2016b).
\newblock Fairness in learning: Classic and contextual bandits.
\newblock In {\em Advances in Neural Information Processing Systems}, pages
  325--333.

\bibitem[Kamiran and Calders, 2009]{kamiran2009classifying}
Kamiran, F. and Calders, T. (2009).
\newblock Classifying without discriminating.
\newblock In {\em Computer, Control and Communication, 2009. IC4 2009. 2nd
  International Conference on}, pages 1--6. IEEE.

\bibitem[Kamishima et~al., 2011]{kamishima2011fairness}
Kamishima, T., Akaho, S., and Sakuma, J. (2011).
\newblock Fairness-aware learning through regularization approach.
\newblock In {\em Data Mining Workshops (ICDMW), 2011 IEEE 11th International
  Conference on}, pages 643--650.

\bibitem[Kleinberg et~al., 2017]{kleinberg2016inherent}
Kleinberg, J.~M., Mullainathan, S., and Raghavan, M. (2017).
\newblock Inherent trade-offs in the fair determination of risk scores.
\newblock In {\em 8th Innovations in Theoretical Computer Science Conference,
  {ITCS} 2017, January 9-11, 2017, Berkeley, CA, {USA}}, pages 43:1--43:23.

\bibitem[Liu et~al., 2017]{liu2017calibrated}
Liu, Y., Radanovic, G., Dimitrakakis, C., Mandal, D., and Parkes, D.~C. (2017).
\newblock Calibrated fairness in bandits.
\newblock {\em arXiv preprint arXiv:1707.01875}.

\bibitem[Pedreshi et~al., 2008]{pedreshi2008discrimination}
Pedreshi, D., Ruggieri, S., and Turini, F. (2008).
\newblock Discrimination-aware data mining.
\newblock In {\em Proceedings of the 14th ACM SIGKDD international conference
  on Knowledge discovery and data mining}, pages 560--568. ACM.

\bibitem[Phillips, 2004]{phillips2004defending}
Phillips, A. (2004).
\newblock Defending equality of outcome.
\newblock {\em Journal of political philosophy}, 12(1):1--19.

\bibitem[Rawls, 2001]{rawls2001justice}
Rawls, J. (2001).
\newblock {\em Justice as fairness: A restatement}.
\newblock Harvard University Press.

\bibitem[Ross, 1996]{ross1996stochastic}
Ross, S. (1996).
\newblock {\em Stochastic processes}.
\newblock Wiley series in probability and statistics: Probability and
  statistics. Wiley.

\bibitem[Sen, 2013]{sen20136}
Sen, A. (2013).
\newblock 6. equality of what?
\newblock {\em Globalization and International Development: The Ethical
  Issues}, page~61.

\bibitem[Sweeney, 2013]{Sweeney2013}
Sweeney, L. (2013).
\newblock Discrimination in online ad delivery.
\newblock {\em Queue}, 11(3):10.

\bibitem[Yona and Rothblum, 2018]{Yona2018}
Yona, G. and Rothblum, G. (2018).
\newblock Probably approximately metric-fair learning.
\newblock In {\em {International Conference on Machine Learning (ICML)}}, pages
  5666--5674.

\bibitem[Zafar et~al., 2017]{zafar2017fairness}
Zafar, M.~B., Valera, I., Gomez~Rodriguez, M., and Gummadi, K.~P. (2017).
\newblock Fairness beyond disparate treatment \& disparate impact: Learning
  classification without disparate mistreatment.
\newblock In {\em Proceedings of the 26th International Conference on World
  Wide Web}, pages 1171--1180.

\bibitem[Zemel et~al., 2013]{zemel2013learning}
Zemel, R., Wu, Y., Swersky, K., Pitassi, T., and Dwork, C. (2013).
\newblock Learning fair representations.
\newblock In {\em International Conference on Machine Learning}, pages
  325--333.

\end{thebibliography}
%\appendix
%\input{appendix.tex}
\end{document}